\documentclass{article}

\usepackage{microtype}
\usepackage{graphicx}
\usepackage{subfigure}
\usepackage{booktabs} 

\usepackage{hyperref}

\usepackage[accepted]{icml2024}

\usepackage{amsmath}
\usepackage{amssymb}
\usepackage{mathtools}
\usepackage{amsthm}

\usepackage[capitalize,noabbrev]{cleveref}

\theoremstyle{plain}
\newtheorem{theorem}{Theorem}[section]

\newtheorem{lemma}{Lemma}
\newtheorem{claim}{Claim}
\newtheorem{corollary}[theorem]{Corollary}
\theoremstyle{definition}

\theoremstyle{remark}

\usepackage[textsize=tiny]{todonotes}

\icmltitlerunning{Offline Imitation from Observation via Primal Wasserstein State Occupancy Matching}

\begin{document}

\twocolumn[
\icmltitle{Offline Imitation from Observation via\\ Primal Wasserstein State Occupancy Matching}



\icmlsetsymbol{equal}{*}

\begin{icmlauthorlist}
\icmlauthor{Kai Yan}{uiuc}
\icmlauthor{Alexander G. Schwing}{uiuc}
\icmlauthor{Yu-Xiong Wang}{uiuc}
\end{icmlauthorlist}

\icmlaffiliation{uiuc}{{\color{black}The Grainger College} of Engineering, University of Illinois Urbana-Champaign, Urbana, Illinois, USA}

\icmlcorrespondingauthor{Kai Yan}{kaiyan3@illinois.edu}

\icmlkeywords{Machine Learning, ICML, Reinforcement Learning, Imitation Learning, Wasserstein Distance}

\vskip 0.3in
]



\printAffiliationsAndNotice{}  

\begin{abstract}
In real-world scenarios, arbitrary interactions with the environment can often be costly, and actions of expert demonstrations are not always available. To reduce the need for both, offline Learning from Observations (LfO) is extensively studied: the agent learns to solve a task given only expert states and \textit{task-agnostic} non-expert state-action pairs. The state-of-the-art DIstribution Correction Estimation (DICE) methods, as exemplified by SMODICE, minimize the state occupancy divergence between the learner's and empirical expert policies. However, such methods are limited to either $f$-divergences (KL and $\chi^2$) or Wasserstein distance with Rubinstein duality, the latter of which constrains the underlying distance metric crucial to the performance of Wasserstein-based solutions. To enable more flexible distance metrics, we propose Primal Wasserstein DICE (PW-DICE). It minimizes the primal Wasserstein distance between the learner and expert state occupancies and leverages a contrastively learned distance metric. Theoretically, our framework is a \textit{generalization} of SMODICE, and is \textit{the first work} that \textit{unifies} $f$-divergence and Wasserstein minimization. Empirically, we find that PW-DICE improves upon several state-of-the-art methods. The code is available at \url{https://github.com/KaiYan289/PW-DICE}.
\end{abstract}

\section{Introduction}
\label{sec:intro}
Recent years have witnessed remarkable advances in \emph{offline} Reinforcement Learning (RL)~\cite{chen2021decisiontransformer,CQL,kostrikov2021iql}: sequential decision-making problems are addressed with independently collected interaction data  rather than an online interaction which is often costly to conduct (e.g., autonomous driving~\cite{autonomous-driving}). Even without online interaction, methods achieve high sample efficiency. Such methods, however, require reward labels that are often missing when data {\color{black} are} collected in the wild~\cite{youtube-navigation}. In addition, an informative reward is expensive to obtain for many tasks, such as robotic manipulation, as it  requires a carefully hand-crafted design~\cite{yu2019meta}.
To bypass the need for reward labels, offline Imitation Learning (IL) has prevailed~\cite{Hakhamaneshi2022FIST,NIPS2016_cc7e2b87,DBLP:conf/iclr/KimSLJHYK22}. It enables the agent to learn from existing demonstrations without reward labels. However, just like reward labels, expert demonstrations are also expensive and often scarce, as they need to be recollected repeatedly for every task of interest. Among different types of expert data shortages, there is one widely studied type: \textit{offline Learning from Observations (LfO).} In LfO, only the expert state, instead of both state and action, is recorded. This setting is useful when learning from experts with different embodiment~\cite{ma2022smodice} or from video demonstrations~\cite{Chen2021LearningGR}, where the expert action is either not applicable or not available. 

Many methods have {\color{black} thus} been proposed in offline LfO, including inverse RL~\cite{Zolna2020OfflineLF, Torabi2018GenerativeAI, DBLP:conf/iclr/KostrikovADLT19}, similarity-based reward labeling~\cite{TCN2017, Chen2021LearningGR}, and action pseudo-labeling~\cite{bco, Kumar2019LearningNS}. The state-of-the-art solution {\color{black} to} LfO is the family of DIstribution Correction Estimation (DICE) methods, which are LobsDICE~\cite{Kim2022LobsDICEOL} and SMODICE~\cite{ma2022smodice}: both methods perform a convex optimization in the dual space to minimize the $f$-divergence of the state occupancy (visitation frequency) between the learner and the empirical expert policies approximated from the dataset. Notably, DICE methods mostly focus on $f$-divergences~\cite{Kim2022LobsDICEOL, ma2022smodice, valuedice, DBLP:conf/iclr/KimSLJHYK22} (mainly KL-divergence and $\chi^2$-divergence; see Appendix~\ref{sec:mathcon} for definition), metrics that ignore some underlying geometric properties of the distributions~\cite{WGANWF}. While there is a DICE {\color{black}variant}, SoftDICE~\cite{sun2021softdice}, that introduces the Wasserstein distance to DICE methods, it adopts the Kantorovich-Rubinstein duality~\cite{KR:58,COTFNT}, which limits the choice of the underlying distance metric: duality requires the underlying metric to be Euclidean~\cite{WGANWF}. This limitation of the distance metric is not only theoretically infavorable, but also impacts practical performance. Concretely, we find the distance metric in Wasserstein-based methods to be crucial for performance (Sec.~\ref{sec:method_motivation}).  

\begin{figure*}[t]
    \centering
    \begin{minipage}[c]{0.42\linewidth}
\subfigure[Problem Setting]{\includegraphics[height=3.5cm]{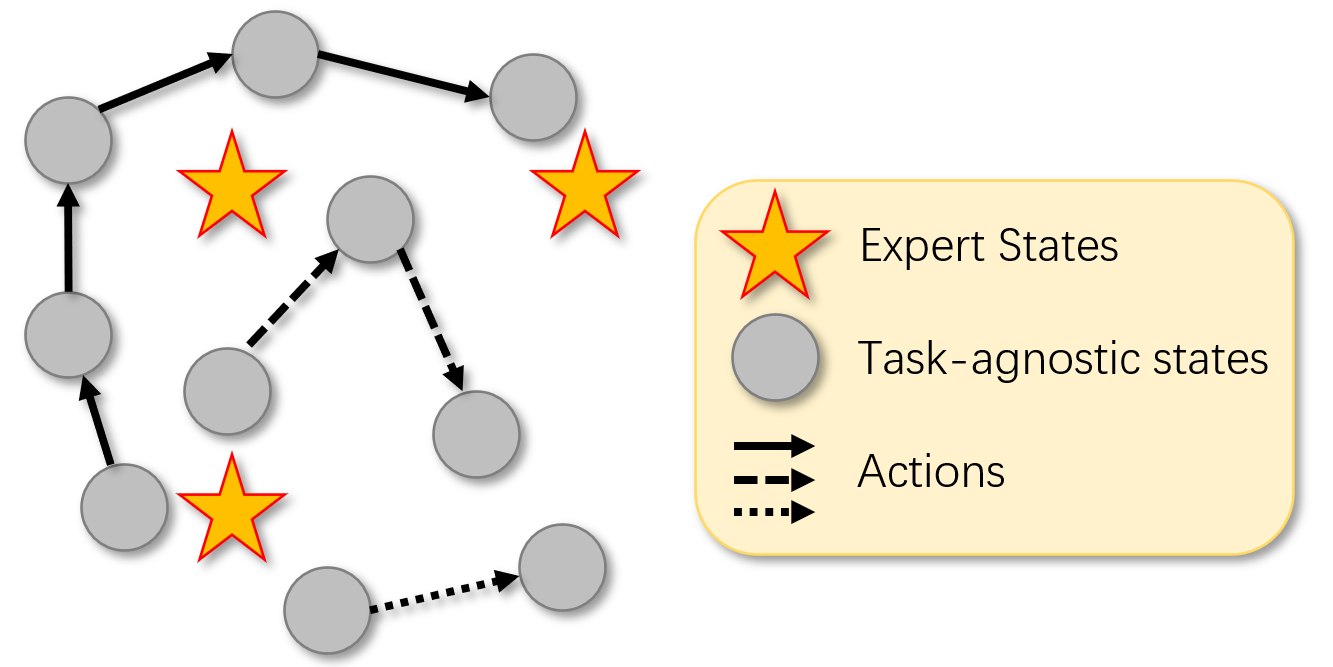}}
\end{minipage}
\begin{minipage}[c]{0.3\linewidth}
\subfigure[Wasserstein Optimization]{\includegraphics[height=3.5cm]{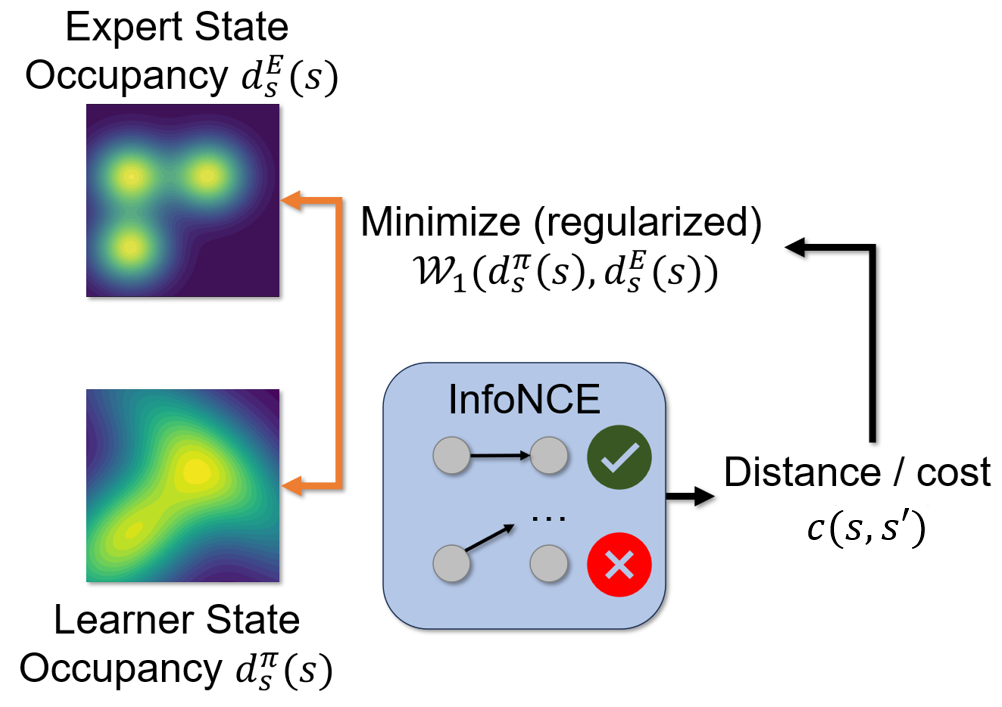}}
\end{minipage}
\begin{minipage}[c]{0.25\linewidth}
\subfigure[Weighted BC]{\includegraphics[height=3.5cm]{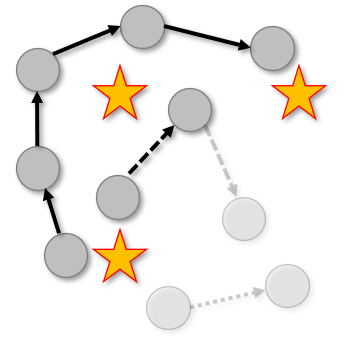}}
\end{minipage}
    \caption{An illustration of our method, PW-DICE. \textbf{a)} Problem setting: different trajectories are illustrated by different styles of arrows. \textbf{b)} PW-DICE minimizes regularized 1-Wasserstein distance between the learner's state occupancy $d^\pi_s(s)$ and the expert state occupancy $d^E_s(s)$. The underlying distance function is contrastively learned to represent the reachability between the states. \textbf{c)} With the matching result, weights are calculated for downstream weighted Behavior Cloning (BC) to retrieve the policy. High transparency indicates a small weight for the state and its corresponding action.}
    \label{fig:teaser}
\end{figure*}

To enable more flexible distance metrics, we propose Primal Wasserstein DICE (PW-DICE), a DICE method that optimizes the primal form of the Wasserstein distance. PW-DICE is illustrated in Fig.~\ref{fig:teaser}. With {\color{black} an} adequate regularizer for offline pessimism~\cite{jin2021pessimism}, the joint minimization of the Wasserstein matching variable and the learner policy can be  formulated as a convex optimization over the Lagrange space.
The policy is then retrieved by weighted behavior cloning with weights determined by the Lagrange function. Different from SMODICE and LobsDICE, {\color{black} our} underlying distance metric can be chosen arbitrarily, and different from all prior work, we explore the possibility of contrastively learning the metric from data. Compared to existing Wasserstein-based work that either uses Rubinstein dual or chooses simple, fixed metrics (e.g., Euclidean for PWIL~\cite{dadashi2020primal} and cosine for OTR~\cite{luo2023otr}), our effort endows PW-DICE with much more flexibility. Meanwhile, with specifically chosen hyperparameters, SMODICE can be obtained as a special case of PW-DICE, which theoretically guarantees our performance.

We summarize our contributions as follows: {\color{black} 1) to our best knowledge, this is \textit{the first work} that sheds light on the practical importance of the underlying distance metric in LfO;} 
2) we propose a novel offline LfO method, PW-DICE, which uses the primal Wasserstein distance for LfO, gaining more flexibility regarding the distance metric than prior work, while removing the assumption for data coverage;
3) we theoretically prove that PW-DICE is a generalization of SMODICE, thus providing \textit{the first unified framework} for Wasserstein-based and $f$-divergence-based DICE methods;
4) we empirically show that our method achieves better results than the state of the art on multiple offline LfO testbeds.
\section{Preliminaries}


\textbf{Markov Decision Process.} The Markov Decision Process (MDP) is a widely adopted formulation for sequential decision-making problems. An MDP has five components: a state space $S$, an action space $A$, a transition function $T$, a reward $r$, and a discount factor $\gamma$. An MDP evolves in discrete steps: at step $t\in\{0,1,2,\dots\}$, {\color{black} the} state $s_t\in S$ is given, and {\color{black} an} agent, {\color{black} following} its policy $\pi(a_t|s_t)\in\Delta(A)$ ({\color{black} where }$\Delta(A)$ is the probability simplex over $A$), chooses an action $a_t\in A$. After receiving $a_t$, the MDP transits to a new state $s_{t+1}\in S$ according to the transition probability function $T(s_{t+1}|s_t,a_t)$, and yields a reward $r(s_t,a_t)\in\mathbb{R}$ as feedback. The agent needs to maximize the discounted total reward $\sum_t\gamma^tr(s_t,a_t)$ with discount factor $\gamma\in[0,1]$. A complete run of the MDP is defined as an episode, with the state(-action) pairs collected along the trajectory $\tau$. The state occupancy, which is the visitation frequency of states given policy $\pi$, is $d^\pi_s(s)=(1-\gamma)\sum_t\gamma^t\Pr(s_t=s)$. See Appendix~\ref{sec:mathcon}  for more rigorous definitions of the state occupancy and other occupancies.

\textbf{Offline Imitation Learning from Observations (LfO).} In offline LfO, the agent needs to learn from two sources of data: 1) the \textit{expert} dataset $E$ with state-only trajectories $\tau_E=\{s_1, s_2,\dots, s_{n_1}\}$ that solve the exact target task, and 2) the \textit{task-agnostic} non-expert dataset $I$ consisting of less relevant state-action trajectories $\tau_I=\{(s_1,a_1),(s_2,a_2),\dots,(s_{n_2},a_{n_2})\}$. Ideally, the agent learns the environment dynamics from $I$, and tries to follow the expert states in $E$ with information about the MDP inferred from $I$. The state-of-the-art methods in offline LfO are SMODICE~\cite{ma2022smodice} and LobsDICE~\cite{Kim2022LobsDICEOL}. The two methods are in spirit similar, with the former minimizing state occupancy divergence and the latter optimizing adjacent \textit{state-pair} occupancy divergence.


\textbf{Wasserstein Distance.} The Wasserstein distance, also known as Earth Mover's Distance (EMD)~\cite{kantorovich1960mathematical}, is widely used as the distance between two probability distributions. It captures the geometry of the underlying space better and does not require any intersection between the support sets. For two distributions $p\in\Delta(S),q\in\Delta(S)$ over state space $S$, the Wasserstein distance\footnote{Unless otherwise specified, we only consider $1$-Wasserstein distance in this paper.}  with {\color{black} an} underlying metric $c(x,y):S\times S\rightarrow\mathbb{R}$ can be written as $\mathcal{W}(p,q)=\inf_{\Pi\in S\times S}\int_{x\in S}\int_{y\in S}\Pi(x,y)c(x,y)$, which is the \textit{primal form} of the Wasserstein distance; {\color{black} $\Pi$ is the matching variable between $p$ and $q$.} Wasserstein also has an equivalent Kantorovich-Rubinstein dual form~\cite{KR:58}, which is $\mathcal{W}(p,q)=\max_{\|f\|_L\leq 1}\mathbb{E}_{x\sim p}f(x)-\mathbb{E}_{y\sim q}f(y)$, where $\|f\|_L\leq 1$ means that the function $f$ is $1$-Lipschitz. While this form is  often adopted by the machine learning community, the Lipschitz constraint is usually implemented by a gradient regularizer {\color{black} in practice}. As the gradient is defined using a Euclidean distance, the underlying distance metric for Rubinstein duality is also restricted to Euclidean~\cite{WGANWF}, which is often suboptimal.
\section{Method}
\label{sec:method}

This section is organized as follows: in Sec.~\ref{sec:method_motivation}, we first validate our motivation, i.e., the importance of selecting an adequate distance metric by comparing metrics using existing Wasserstein-based solutions; then, we detail our proposed optimization objective in Sec.~\ref{sec:reg}; 
finally, we discuss our choice of {\color{black} the} distance metric in Sec.~\ref{sec:contra}. See Tab.~\ref{tab:notelist} in Appendix~\ref{sec:notlist} for a reference of the notation, and Appendix~\ref{sec:proof} for detailed derivations.
\begin{figure}[t]
    \centering
    \includegraphics[width=\linewidth]{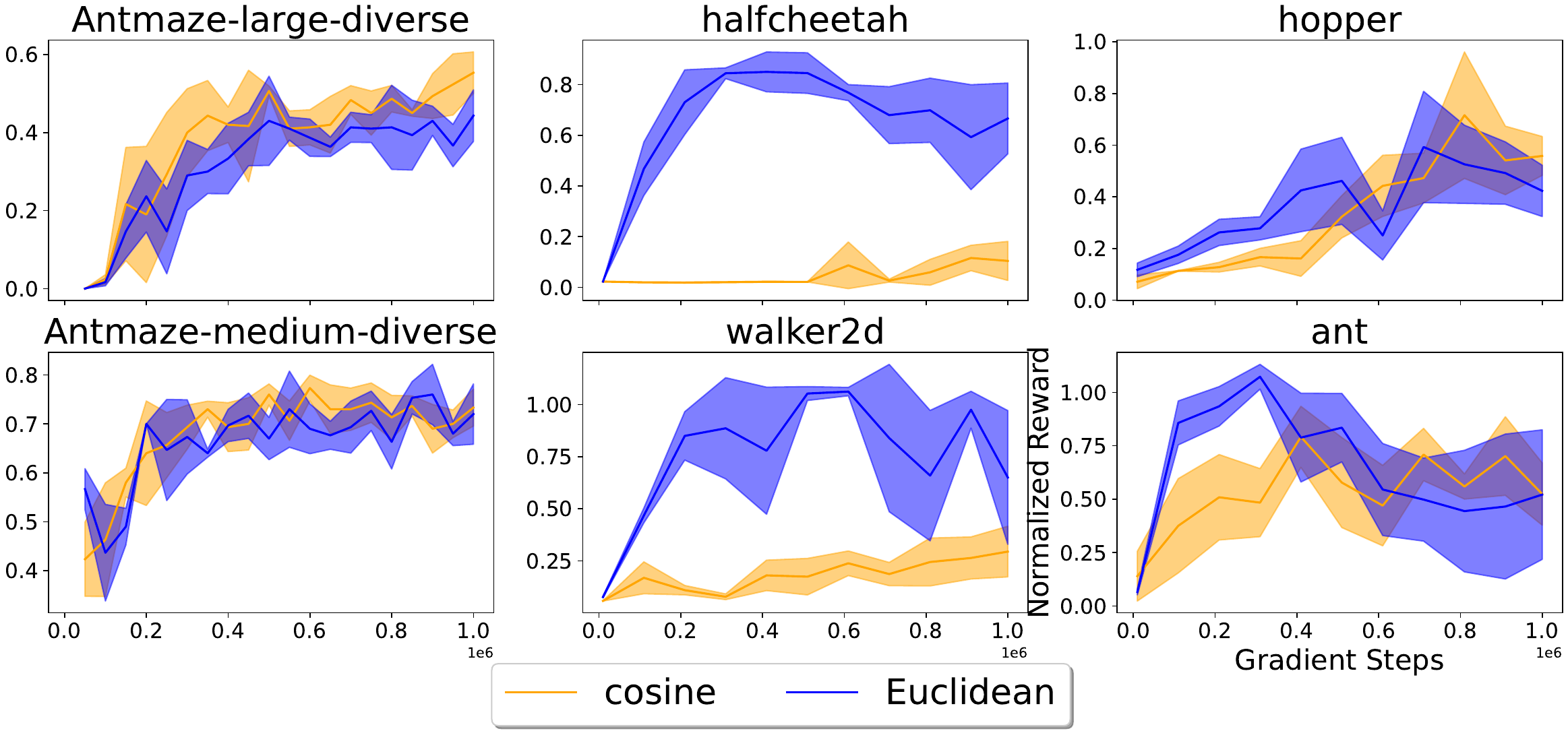}
    \caption{Performance comparison between {\color{black}the} default (normalized cosine) distance metric and Euclidean distance metric using OTR~\cite{luo2023otr}  (first column), and SMODICE~\cite{ma2022smodice} (second and third column{\color{black}s}). The result shows that the underlying distance metric is crucial for the {\color{black}performance} of Wasserstein-based methods.}
    \label{fig:motivation}
\end{figure}
\subsection{Validation of Motivation}
\label{sec:method_motivation}

As mentioned in Sec.~\ref{sec:intro}, our goal is to improve the idea of divergence minimization between the learner's policy and the expert policy estimated from the expert dataset. For this we suggest to use the primal Wasserstein distance {\color{black}which allows flexibly using an} arbitrary underlying distance metric. To show the importance of distance metrics and the advantage of being able to select them, we study Optimal Transport Reward (OTR)~\cite{luo2023otr}, a current Wasserstein-based IL method that can be applied to our LfO setting. OTR optimizes the primal Wasserstein distance between every trajectory in the task-agnostic dataset and the expert trajectory, and uses the result to assign a reward to each state in the task-agnostic dataset. Then, offline RL is applied to retrieve the optimal policy. Fig.~\ref{fig:motivation} shows results of OTR on the D4RL MuJoCo dataset (see Sec.~\ref{sec:MuJoCoexp} for more) with testbeds appearing in both SMODICE~\cite{ma2022smodice} and OTR. We test both the cosine-similarity-based occupancy used in~\citet{luo2023otr} and the Euclidean distance as the underlying distance metric. The results illustrate that distance metrics have a significant impact on {\color{black}outcomes}. Thus, {\color{black}selecting} a good metric is crucial for the performance of Wasserstein-based solutions.  

This experiment motivates our desire to develop a method that permits {\color{black}the} use of arbitrary distance metrics in Wasserstein-based formulations. Further, the {\color{black} observed} performance difference inspires us to automate {\color{black} the} selection of the metric{\color{black}, going} beyond classic metrics such as cosine and Euclidean. We discuss the formulation of our suggested method next.

\subsection{Optimization Objective}
\label{sec:reg}

Our goal is to optimize the primal Wasserstein distance between the model-estimated state occupancy $d^\pi_s(s)$ induced by the policy $\pi$ and the empirical state occupancy  $d^E_s(s)$ estimated from expert data. 
This can be formalized via the following program: 
\begin{equation}
\min_{\Pi, \pi}\sum_{s_i\in S}\sum_{s_j\in S}\Pi(s_i, s_j)c(s_i,s_j), \text{s.t. }  d^\pi_{sa}\geq 0, \Pi\geq 0; 
\label{eq:primal_main}
\end{equation}
\begin{gather*}
\forall s\in S, d^\pi_s(s)=(1-\gamma)p_0(s)+\gamma\sum_{\bar{s},\bar{a}}d^\pi_{sa}(\bar{s},\bar{a})p(s|\bar{s},\bar{a});\\
\scalebox{0.95}{$\forall s_i, s_j\in S,\sum_k\Pi(s_k, s_j)=d^E_s(s_j), \sum_k\Pi(s_i,s_k)=d^\pi_s(s_i).$}
\end{gather*}
In Eq.~\eqref{eq:primal_main}, we use $\Pi(s_i, s_j)$ as the matching variable between the two state occupancy distributions $d^\pi_s(s)$ and $d^E_s(s)$, and $c(s_i, s_j)$ is the distance between states $s_i$ and $s_j$. 
Further, $d^\pi_{sa}$ is the state-action occupancy of our learned policy $\pi$, and $p_0\in\Delta(S)$ is the distribution of the MDP's initial states. Note that there are two types of constraints in Eq.~\eqref{eq:primal_main}: the first row, together with $d^\pi_{sa}\geq 0$, is the marginal constraint for the matching variable $\Pi$. The second row and $\Pi\geq 0$ are the \textit{Bellman flow constraints}~\cite{ma2022smodice} that ensure correspondence between $d^\pi_s$ and a feasible policy $\pi$.

For a tabular MDP, Eq.~\eqref{eq:primal_main} can be solved by any Linear Programming (LP) solver, as both the objective and the constraints are linear. However, using an LP solver is impractical for any MDP with continuous state or action spaces. In these cases, we need to convert the problem into a program that is easy to optimize. 
The common way to remove constraints is to consider the Lagrangian dual problem. However, the Lagrangian dual problem of an LP with constraints is also an LP with constraints. In order to reduce constraints in the dual program, we smooth the objective by using: 
\begin{equation}
\begin{aligned}
&
\Pi(s_i,s_j)c(s_i,s_j)+\epsilon_1 D_f(\Pi\|U)+\epsilon_2 D_f(d^\pi_{sa}\|d^I_{sa}).
\label{eq:reg}
\end{aligned}
\end{equation}

In Eq.~\eqref{eq:reg}, $d^I_s$ and $d^I_{sa}$ are the empirical state occupancy and state-action occupancy of the task-agnostic dataset $I$ respectively. Further, we let $U(s,s')=d^E_s(s)d^I_s(s')$, i.e., $U$ is the product of two independent distributions $d^E_s$ and $d^I_s$. Moreover,   $\epsilon_1>0, \epsilon_2>0$ are hyperparameters, and $D_f$ can be any $f$-divergence (we will focus on {\color{black}the} KL-divergence in this paper). Note, despite the use of an $f$-divergence, different from SMODICE~\cite{ma2022smodice} or LobsDICE~\cite{Kim2022LobsDICEOL}, this formulation does not require data coverage of the task-agnostic data over expert data. The two regularizers are ``pessimistic'': they encourage the agent to stay within the support set of the dataset which is common in offline IL/RL~\cite{jin2021pessimism}.

Given the  smooth objective, we  apply 
Fenchel duality to derive a  robust single-level optimization in the dual space. See Appendix~\ref{sec:optdetail} for a detailed derivation. 

The dual program  when letting $D_f$ refer to a KL-divergence (see Appendix~\ref{sec:chidiv} for a discussion on $\chi^2$-divergence) reads as follows: 
%
%
\begin{equation}
\small
\begin{aligned}
\label{eq:final}
&\min_\lambda\ \epsilon_1\log \mathbb{E}_{s_i\sim I, s_j\sim E}\exp\left(\frac{\lambda_{i+|S|}+\lambda_{j+2|S|}-c(s_i,s_j)}{\epsilon_1}\right)\\
+&\epsilon_2\log \mathbb{E}_{(s_i,a_j)\sim I}\exp\left(\frac{-\gamma \mathbb{E}_{s_k\sim p(\cdot|s_i,a_j)}\lambda_k+\lambda_i-\lambda_{i+|S|}}{\epsilon_2}\right)\\&-\left[(1-\gamma)\mathbb{E}_{s\sim p_0}\lambda_{:|S|}+\mathbb{E}_{s\sim E}\lambda_{2|S|:3|S|}\right].
\end{aligned}
\end{equation}
Intuitively, the dual variables $\lambda\in\mathbb{R}^{3|S|}$ in the objective are divided into three parts, each of size $|S|$: $\lambda_i, i\in\{1,2,\dots,|S|\}$ can be seen as a variance of the \textit{value function}, where $-\gamma \mathbb{E}_{s_k\sim p(\cdot|s_i,a_j)}\lambda_k+\lambda_i-\lambda_{i+|S|}$ is its Bellman residual, or negative $\text{TD}(0)$ advantage. $\lambda_{i+|S|}$ and $\lambda_{i+2|S|}$ are costs attached to a particular state: if we compare Wasserstein matching to shipping probability mass, $\lambda_{i+|S|}$ would be the loading cost from states of $\pi$ and $\lambda_{i+2|S|}$ would be the unloading cost to the states of $E$. 
Note, by Theorem~\ref{thm:smodice} (see Appendix~\ref{sec:optdetail} for {\color{black}a} detailed derivation), we have $d^\pi_{sa}=d^I_{sa}\cdot \text{softmax}\left(\frac{-\gamma \mathbb{E}_{s_k\sim p(\cdot|s_i,a_j)}\lambda_k+\lambda_i-\lambda_{i+|S|}}{\epsilon_2}\right)$ at the optimum, and the denominator of the softmax is summing over all state-action pairs. 

With $\lambda$ optimized, we retrieve the desired policy $\pi$ by weighted behavior cloning, maximizing the following objective:
\begin{equation}
\small
\begin{aligned}
\label{eq:policy}
&\mathbb{E}_{(s_i,a_j)\sim d^\pi_{sa}}\log \pi(a|s)=\mathbb{E}_{(s_i,a_j)\sim I}\frac{d^\pi_{sa}(s_i,a_j)}{d^I_{sa}(s_i,a_j)}\log \pi(a_j|s_i)\\
\propto\ & \mathbb{E}_{(s_i,a_j)\sim I}\exp\left(\frac{-\gamma \mathbb{E}_{s_k}\lambda_k+\lambda_i-\lambda_{i+|S|}}{\epsilon_2}\right)\log \pi(a_j|s_i).
\end{aligned}
\end{equation}
In practice, we use $1$-sample estimation for $p(\cdot|s_i, a_j)$, {\color{black}a method} found to be simple and effective {\color{black} in prior work}~\cite{ma2022smodice, Kim2022LobsDICEOL}. That is, we sample $(s_i,a_j,s_k)\sim I$ from the dataset instead of $(s_i, a_j)$, and use $\lambda_k$ corresponding to $s_k$ as an estimation for $\mathbb{E}_{s_k\sim p(\cdot|s_i,a_j)}\lambda_k$. Since the number of states can be infinite in practice, we use a 3-head neural network to estimate $\lambda_s, \lambda_{s+|S|}$ and $\lambda_{s+2|S|}$ given state $s$. See Appendix~\ref{sec:pseudocode} for pseudo-code of our algorithm where we iteratively optimize the dual by Eq.~\eqref{eq:final} and obtain the policy $\pi$ by Eq.~\eqref{eq:policy}. 

Importantly, note that our formulation can be seen as a generalization of SMODICE~\cite{ma2022smodice}. It is not hard to see why and we point this out next. SMODICE's objective with KL divergence reads as follows: 
\begin{equation}
\begin{aligned}
\min_V &\log E_{(s,a)\sim I}\left[\exp(R(s)+\gamma \mathbb{E}_{s'\sim(s,a)}-V(s))\right]+\\
&(1-\gamma)\mathbb{E}_{s\sim p_0}\left[V(s)\right],
\end{aligned}
\end{equation}
where $V(s)$ is a value function and $R(s)$ is the reward assigned for states. It is easy to see that $\lambda_i$ corresponds to $V(s)$ in SMODICE, and SMODICE is a special case of PW-DICE with $\epsilon_1\rightarrow 0, \epsilon_2=1$, and $c(s,s')=-R(s)$. We highlight that in the SMODICE setting, the distance $c(s,s')$ only depends on the first state $s$. Thus, the total matching cost is fixed for any matching plan given particular {\color{black}state occupancy of} $\pi$, i.e., $d^\pi_s$. 
Meanwhile, the pessimistic regularizer with large coefficient $\epsilon_2$ dominates. This generalization property also holds for other divergences such as $\chi^2$. See Appendix~\ref{sec:proof} for a more rigorous derivation.

\subsection{Underlying Distance Metric}
\label{sec:contra}

Given Eq.~\eqref{eq:final} and Eq.~\eqref{eq:policy}, it remains to choose the distance metric $c(s_i, s_j)$. For tabular cases, one could use the simplest distance, i.e., $c(s_i,s_j)=1$ if $s_i\neq s_j$, and $0$ otherwise. However, such {\color{black}a} distance only provides ``sparse'' information in the continuous case. The distance will mostly be $0$, and will degrade to all zeros if there is no common state in the expert dataset $E$ and the task-agnostic dataset $I$. To address this, {\color{black}prior work has} explored many heuristic choices, such as cosine similarity~\cite{luo2023otr} or a Euclidean~\cite{sun2021softdice} distance. However, such  choices are often suboptimal for particular environments, as shown when validating our motivation in Sec.~\ref{sec:method_motivation}. 

In this work, inspired by both CURL~\cite{laskin2020curl} and SMODICE~\cite{ma2022smodice}, we propose a weighted sum of $R(s)=\log\frac{d^E_s(s)}{(1-\alpha)d^I_s(s)+\alpha d^E_s(s)}$ and the Euclidean distance between an embedding learned by the InfoNCE~\cite{infonce} loss. To be more specific, we let  the distance metric $c$  be
\begin{equation}
    c(s_i,s_j)=R(s_i)+\beta\|g(s_i)-g(s_j)\|^2_2,
\end{equation} 
where $g(s_i), g(s_j)$ are learned embeddings for the states $s_i, s_j$ respectively, $\alpha$ is a positive constant close to $0$, and $\beta\geq 0$ is a hyperparameter. 

The distance function consists of two parts. The first part, $R(s_i)$, is a modified version of the SMODICE reward function $\log\frac{d^E_s(s)}{d^I_s(s)}$. Intuitively, high $\log\frac{d^E_s(s)}{d^I_s(s)}$ indicates that the state $s$ is more frequently visited by the expert than agents generating the task-agnostic data, which is probably desirable. Such reward can be obtained by training a discriminator {\color{black}$h(s)$} that takes expert states from $E$ as label $1$ and non-expert ones as label $0$. If {\color{black}$h$ is optimal, i.e., $h(s)=h^*(s)=\frac{d^E_s(s)}{d^E_s(s)+d^I_s(s)}$, then we have $\frac{d^E_s(s)}{d^I_s(s)}=\log\frac{h^*(s)}{1-h^*(s)}$.} Based on this, we change the denominator $d^E_s(s)$ to $(1-\alpha)d^I_s(s)+\alpha d^E_s(s)$ to lift the theoretical assumption that the task-agnostic dataset $I$ covers the expert dataset $E$, i.e., $d^I_s(s)>0$ wherever $d^E_s(s)>0$. 

The second part uses the embedding $g(s)$ learned with InfoNCE~\cite{infonce}, which {\color{black}is} also {\color{black}adopted} in CURL~\cite{laskin2020curl} and FIST~\cite{Hakhamaneshi2022FIST}. Different from CURL, where the contrastive learning is an auxiliary loss {\color{black}in addition to} RL for better extraction of features, and FIST, which tries to find the similarity between the current state and a state in the dataset, we want $g(s)$ and $g(s')$ to be similar if and only if they are reachable along trajectories in the task-agnostic dataset. 
For this, we sample a batch of consecutive state pairs $(s_i,s_i'), i\in\{1,2,\dots\}$, and use the following loss function:
\begin{equation}
\small
\label{eq:infonce}
\log\frac{\exp\left(g(s_i)^TWg(s'_i)\right)}{\exp\left(g(s_i)^TWg(s'_i)\right)+\sum_{j\neq i}\exp\left(g(s_i)^TWg(s'_j)\right)}.
\end{equation}
Here, $g(s_i)$ can be seen as an \textit{anchor} in contrastive learning, $W$ is a learned matrix, $g(s'_i)$ is its \textit{positive key}, and $g(s'_j), j\neq i$ is its \textit{negative key}. 
Intuitively, the idea is to learn a good embedding space where the vicinity of a state can be assessed by the Euclidean distance between the embedding vectors. We define the vicinity as the ``reachability'' between states: if one state can reach the other through a trajectory in the task-agnostic data, then states should be close, otherwise they are far from each other. This definition groups states that lead to success  in the embedding space (see Fig.~\ref{fig:visualization} for a visualization), while being robust to actual numerical values of the state (see Sec.~\ref{sec:MuJoCoexp} for empirical evaluations).


\section{Experiments}

We evaluate PW-DICE  across multiple environments. We strive to answer two main questions: 1) can the Wasserstein objective indeed lead to a closer match between the learner's and the expert policies (Sec.~\ref{sec:tabular})?; and 2) can PW-DICE improve upon $f$-divergence based methods on more complicated environments, and does a flexible underlying distance metric indeed help (Sec.~\ref{sec:MuJoCoexp})?

\subsection{Primal Wasserstein vs.\ $f$-Divergence}
\label{sec:tabular}
\textbf{Baselines.} We compare to the two major state-of-the-art baselines, 
SMODICE~\cite{ma2022smodice} and LobsDICE~\cite{Kim2022LobsDICEOL}. We test two variants of our method: 1) Linear Programming (LP) by directly solving Eq.~\eqref{eq:primal_main}; and 2) Regularized (Reg) which solves Eq.~\eqref{eq:reg}. As the environment is tabular, all methods are implemented with CVXPY~\cite{cvxpy}\footnote{\color{black}In our experiments, CVXPY usually invokes Gurobi~\cite{gurobi} for linear programming and MOSEK~\cite{mosek} for other objectives during optimization.} for optimal numerical solutions. The mean and standard deviation  are obtained from $10$ independent runs with different seeds. We evaluate all methods with the \textbf{regret}, i.e., the gap between reward{\color{black}s} gained by {\color{black}the} learner{\color{black}'s and expert policies} (\textit{lower is better}). To be consistent with LobsDICE, in  Appendix~\ref{sec:sup-tabular}, we also compare the Total Variation (TV) distance {\color{black}for} the state and state-pair occupancies, i.e., $\text{TV}(d^\pi_{s}\|d^E_{s})$ and  $\text{TV}(d^\pi_{ss}\|d^E_{ss})$. 

\begin{figure*}[t]
    \centering
    \includegraphics[width=\linewidth]{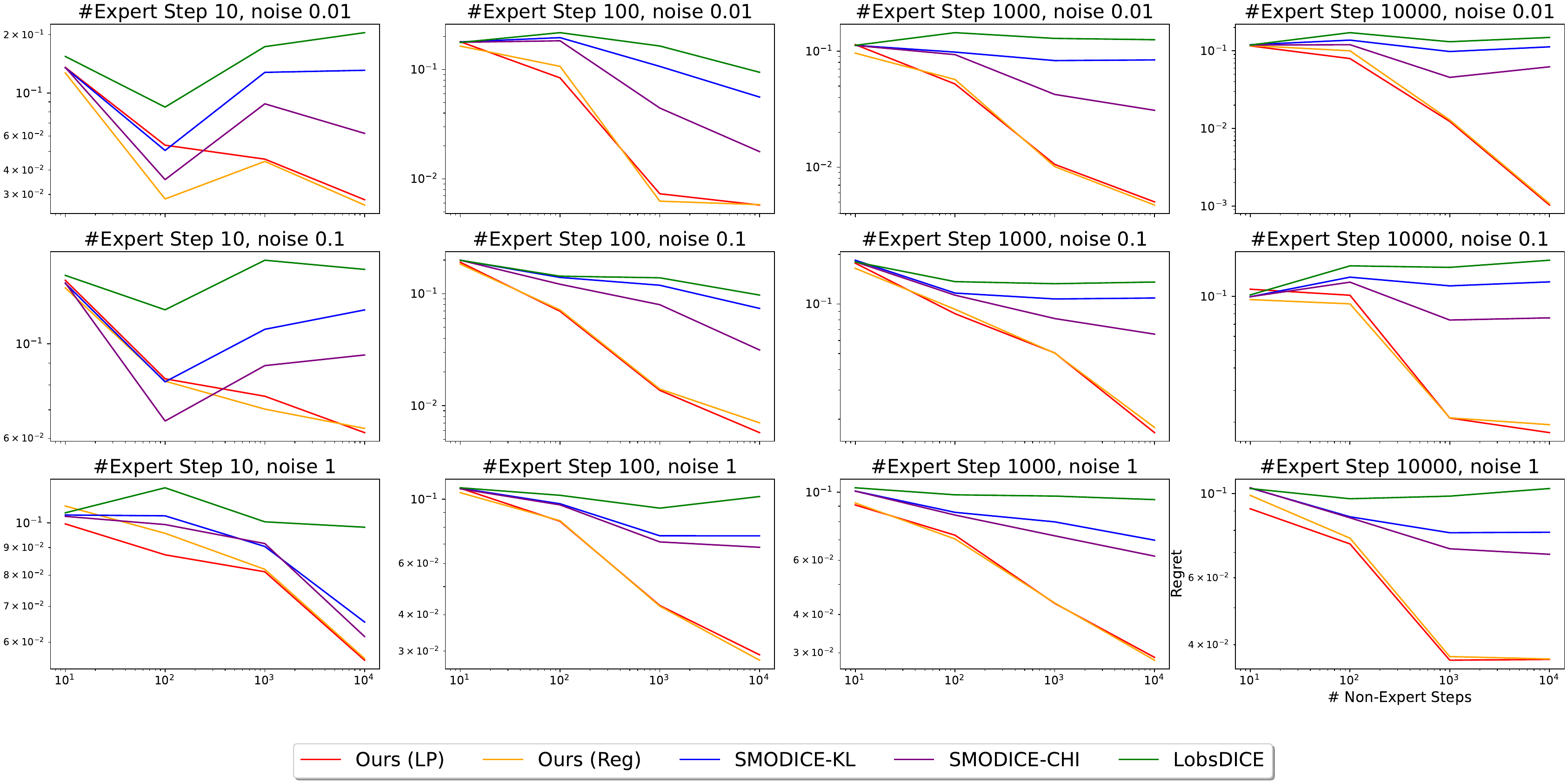}
    \caption{The regret (reward gap between learner and expert) of each method on a tabular environment. We observe our method to work {\color{black}the} best, regardless of the presence of a regularizer; {\color{black} the regularizer is more important in continuous MDPs.}}
    \label{fig:regret-tabular}
\end{figure*}

\textbf{Environment Setup.} Following the random MDP experiment in LobsDICE~\cite{Kim2022LobsDICEOL}, we randomly generate {\color{black}an} MDP with $|S|=20$ states, $|A|=4$ actions, and $\gamma=0.95$. The stochasticity of the MDP is controlled by {\color{black}$\eta\in[0, 1]$}, where $\eta=0$ is deterministic and $\eta=1$ is highly stochastic. Agents always start from one particular state, and aim to reach another particular state with reward $+1$, which is the only source of reward. We report the regret for different $\eta$, expert dataset sizes, and task-agnostic dataset sizes. The only difference from LobsDICE {\color{black}is}: the expert policy is deterministic instead of being softmax, as we found the high connectivity of the MDP states to lead to a near-uniform value function. Thus, the softmax expert policy is highly suboptimal and near-uniform. See Appendix~\ref{sec:expdetail} for {\color{black}explanation} and Appendix~\ref{sec:sup-tabular} for results.


\textbf{Experimental Setup.} As the environment is tabular, as mentioned above, we use CVXPY~\cite{cvxpy} to solve for the optimal policy for each method using the \textit{primal} formulation. For example, we directly solve Eq.~\eqref{eq:primal_main} to get the learner's policy $\pi$. Following SMODICE, for estimating transition function and the task-agnostic average policy $\pi^I$, we simply count the state-action pair and transitions from the task-agnostic dataset $I$, i.e., the transition probability $p(s'|s,a)=\frac{\#[(s,a,s')\in I]}{\#[(s,a)\in I]}$, and $\pi^I(a|s)=\frac{\#[(s,a)\in I]}{\#[s\in I]}$ ($\#$ stands for ``the number of''). Similarly, the expert state occupancy $d^E_s$ is estimated by $d^E_s(s)=\frac{\#[s\in E]}{|E|}$, where $|E|$ is the size of the expert dataset $E$.
Notably, if the denominator is $0$, the distribution will be estimated as uniform. {\color{black}As the environment is tabular, we use the simplest distance metric, described in the beginning of Sec.~\ref{sec:contra}, i.e., $c(s_i,s_j)=1$ if $s_i\neq s_j$ and $0$ otherwise.}  

\vspace{6pt}

\textbf{Main Results.} Fig.~\ref{fig:regret-tabular} shows the regret of each method. We observe our method with or without regularizer to perform similarly  and to achieve the lowest regret across expert dataset sizes in $\{10, 100, 1000, 10000\}$, task-agnostic (non-expert) dataset sizes in $\{10, 100, 1000, 10000\}$, and noise levels $\eta\in\{0.01, 0.1, 1\}$. The gap increases with the task-agnostic dataset size, which shows that our method works better when the MDP dynamics are more accurately estimated. LobsDICE struggles in this scenario, albeit being the best in minimizing the divergence to the softmax expert, which is more stochastic and suboptimal (see Appendix~\ref{sec:sup-tabular} for details). 

\subsection{More Complex Environments}
\label{sec:MuJoCoexp}

\begin{figure*}[t]
    \centering
    \includegraphics[width=\linewidth]{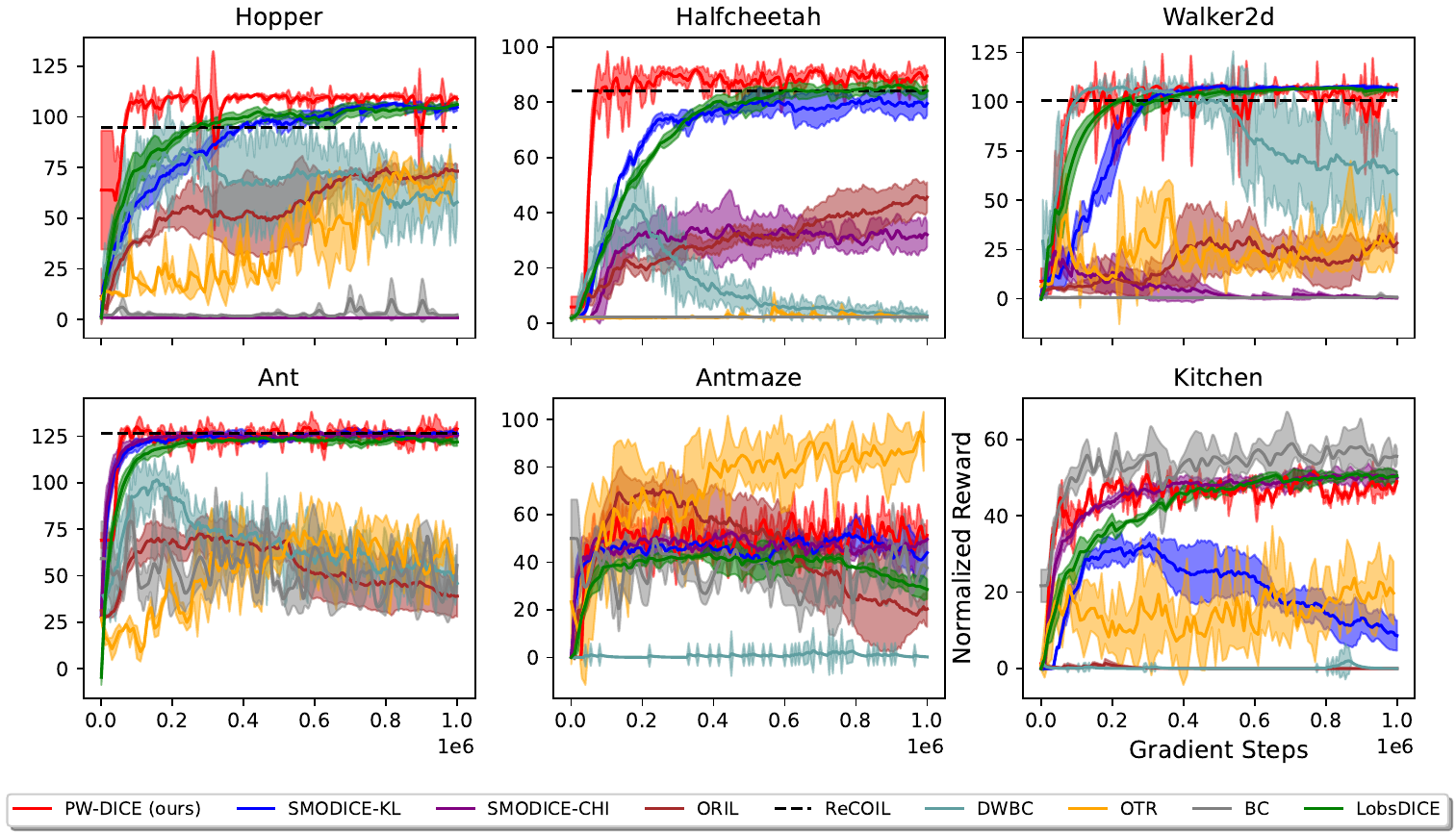}
    \caption{Performance comparison on the MuJoCo testbed. SMODICE-KL and SMODICE-CHI stand for variants of SMODICE using different $f$-divergences (KL or $\chi^2$). Our method generally works the best {\color{black}(i.e., has the highest normalized reward)} among all baselines.}
    \label{fig:MuJoCo-main}
\end{figure*}

\textbf{Baselines.} We adopt seven baselines in our study: state-of-the-art DICE methods SMODICE~\cite{ma2022smodice}, LobsDICE~\cite{Kim2022LobsDICEOL}, and ReCOIL~\cite{Sikchi2023ImitationFA}, non-DICE method ORIL~\cite{Zolna2020OfflineLF}, Wasserstein-based method OTR~\cite{luo2023otr}, DWBC~\cite{DWBC} with extra access to the expert action, and the plain Behavior Cloning (BC). As we have no access to the ReCOIL code, we directly report the final numbers from their paper. Mean and standard deviation  are obtained from $3$ independent runs with different seeds. We measure the performance using the average reward (\textit{higher is better}).

\textbf{Environment and Experimental Setup.} Following SMODICE~\cite{ma2022smodice}, we test PW-DICE on four standard OpenAI gym MuJoCo environments: hopper, halfcheetah, ant, and walker2d, as well as two more challenging MuJoCo testbeds, antmaze and Franka kitchen. The datasets that we use are identical to those in SMODICE (see Appendix~\ref{sec:expdetail} for details). The metric we use is the normalized average reward, where higher reward indicates better performance\footnote{We use the same normalization standard as  D4RL~\cite{fu2020d4rl} and SMODICE~\cite{ma2022smodice}.}. If the final reward is similar, the algorithm with fewer gradient step updates is better. We plot the reward curve, which illustrates the change of the mean and standard deviation of the reward with the number of gradient steps. See Appendix~\ref{sec:expdetail} for hyperparameters.
 
\textbf{Main Results.} Fig.~\ref{fig:MuJoCo-main} shows the results on the MuJoCo testbed, where our method achieves {\color{black}performance} comparable to or better than baselines on all four testbeds. SMODICE with KL-divergence and LobsDICE work decently well, while the other methods struggle.

Note, OTR~\cite{luo2023otr} struggles on most environments despite using {\color{black} the} primal Wasserstein distance, which is probably because the assigned reward calculated by the Wasserstein distance is not always reasonable. See Fig.~\ref{fig:otr-illu} for examples.

\begin{figure}[t]
    \centering
    \subfigure[hopper (success)]{
    \begin{minipage}[b]{0.47\linewidth}
    \includegraphics[width=\linewidth]{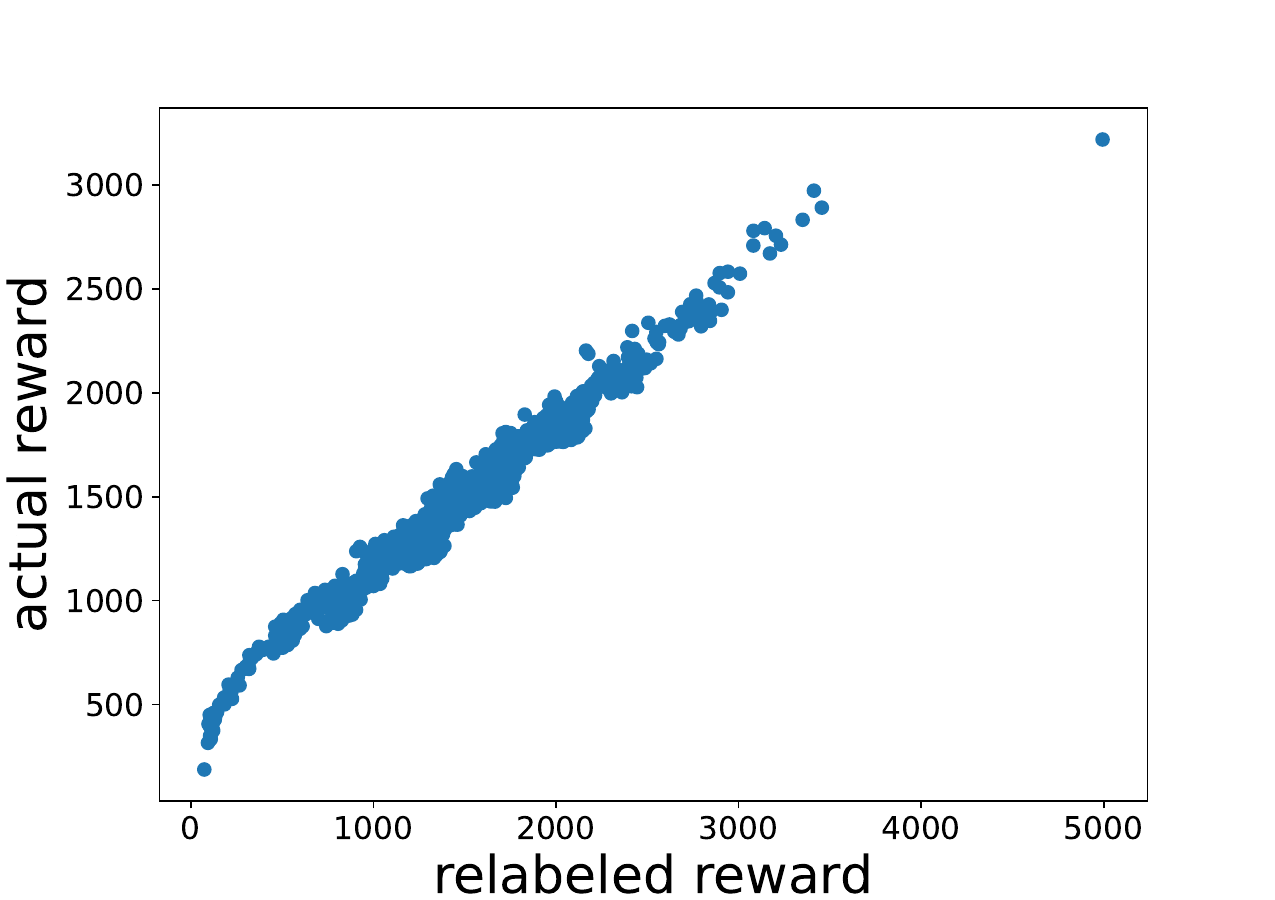}
    \end{minipage}
    }
    \subfigure[halfcheetah ({\color{black}failure})]{
    \begin{minipage}[b]{0.47\linewidth}
    \includegraphics[width=\linewidth]{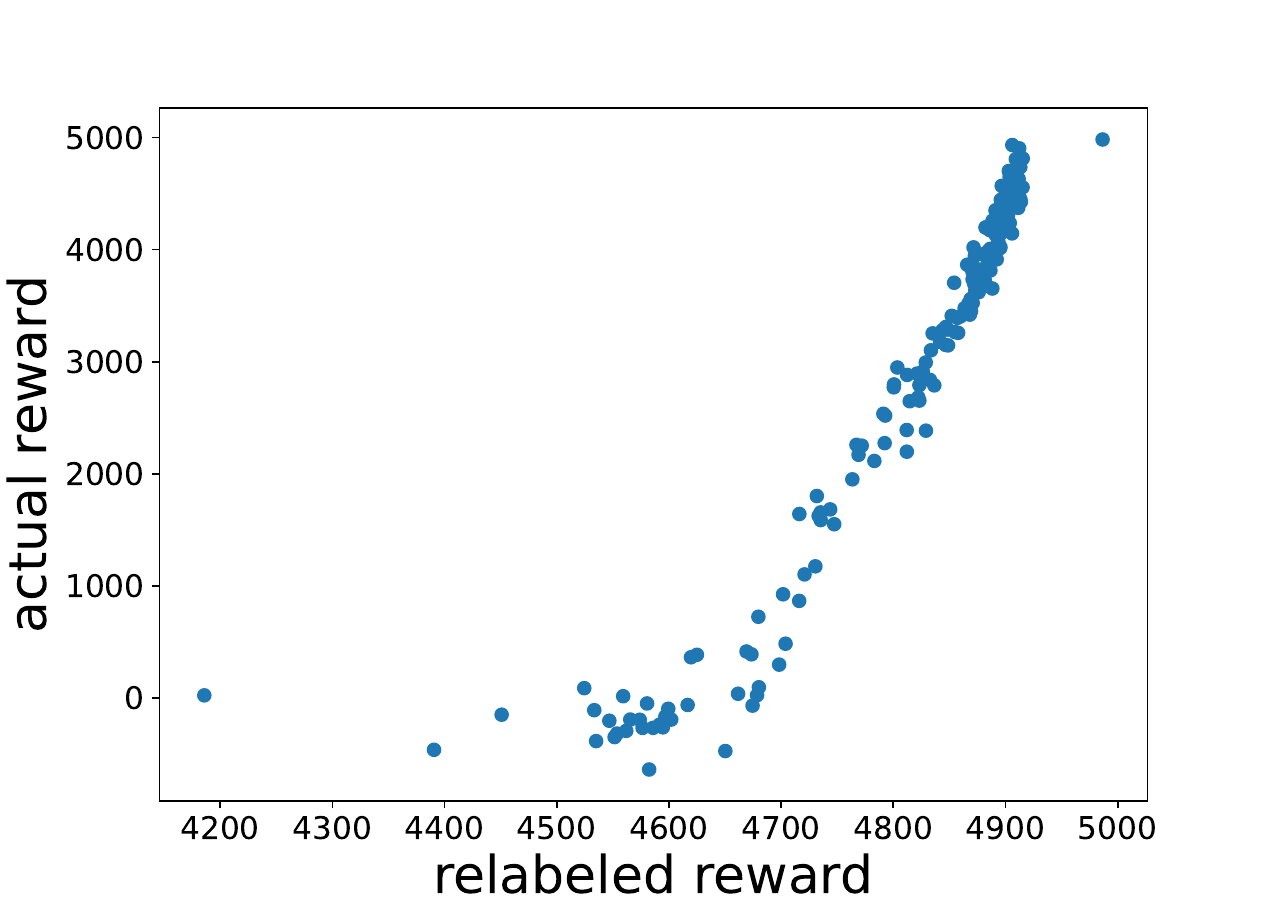}
    \end{minipage}
    }
    \caption{An illustration of successful (coherent with the OTR paper) and failing reward assignment in OTR~\cite{luo2023otr}. OTR performs Wasserstein matching between uniform distributions over {\color{black}the} states of each trajectory in the task-agnostic dataset and the expert dataset, instead of between policy distributions. The reward is calculated from the matching result. Such a solution may fail to differentiate good and bad trajectories by giving similar rewards, {\color{black} as shown in the failure case \textbf{b)}}.}
    \label{fig:otr-illu}
\end{figure}


\textbf{Is our design of distance metric useful?} We illustrate the importance and effectiveness of our distance metric design through qualitative and quantitative studies. 
For a qualitative evaluation, we draw $4$ different trajectories from the D4RL dataset of the MuJoCo hopper environment, and compare the t-SNE~\cite{van2008visualizing} visualization result (for better readability, we only plot $150$ steps of the trajectory). The result in Fig.~\ref{fig:visualization} shows that our embedding successfully learns the topology of reachability between the states, which separates different trajectories and connects states in the same trajectory irrespective of their distance.  

\begin{figure}[t]
    \centering
    \subfigure[State]{
    \begin{minipage}[b]{0.47\linewidth}
    \includegraphics[width=\linewidth]{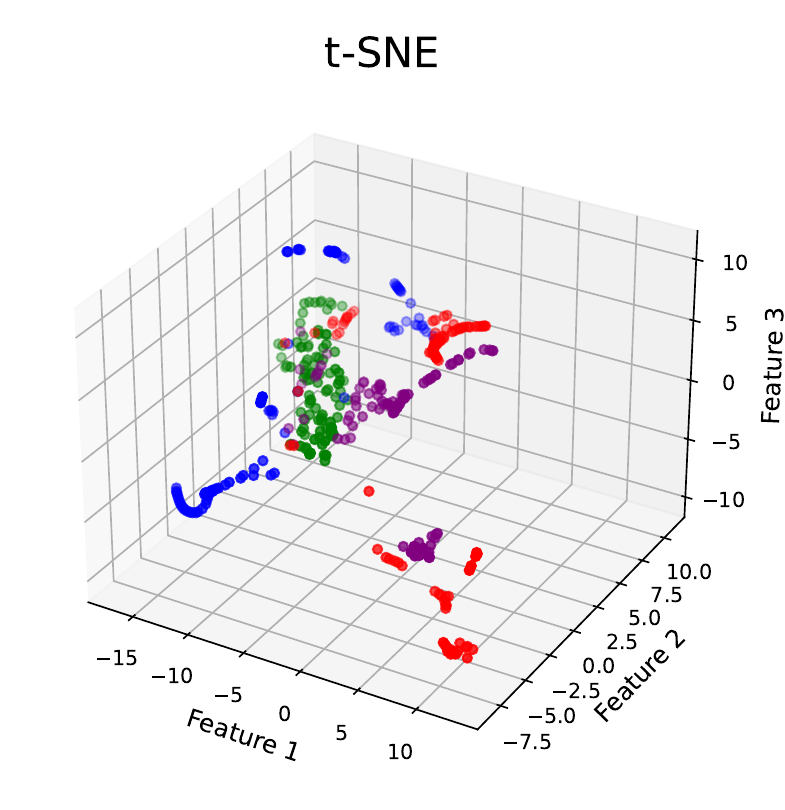}
    \end{minipage}
    }
    \subfigure[Contrastive Embedding]{ 
    \begin{minipage}[b]{0.47\linewidth}
    \includegraphics[width=\linewidth]{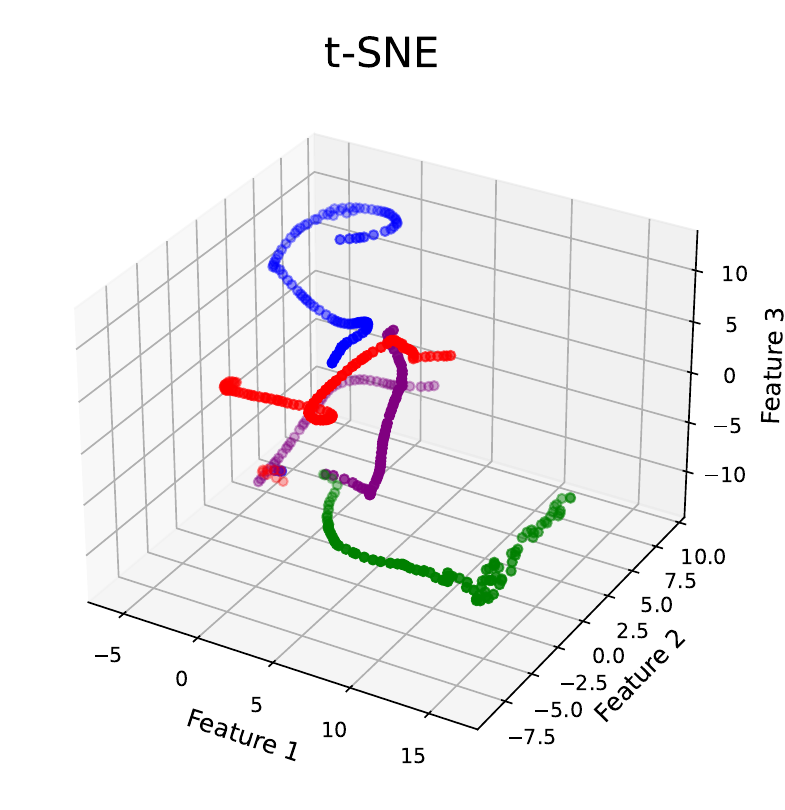}
    \end{minipage}
    }
    \caption{The t-SNE visualization of {\color{black}\textbf{a)}} the states and {\color{black}\textbf{b)}} their corresponding contrastive embeddings. Different colors stand for different trajectories. Note {\color{black}that} the trajectories in the embedding space {\color{black}are} separated despite their proximity in the original state space, and states along the same trajectory are connected despite being separated in the original state space.}
    \label{fig:visualization}
\end{figure}

For a quantitative evaluation, we conduct an ablation study on the distance metric used in PW-DICE. Specifically, we test the result of PW-DICE with $c(s,s')=R(s)$, $c(s,s')=\|s-s'\|^2_2$ (Euclidean), $c(s,s')=1-\frac{s^Ts'}{|s||s'|}$ (cosine similarity), $c(s,s')$ from contrastive learning and their combinations. The result is illustrated in Fig.~\ref{fig:ablation-dist-MuJoCo}. The result shows that both our design of distance and the combination of cosine similarity and $R(s)$ works well, while distance metrics with a single component fail (including Euclidean distance implied by Rubinstein duality). 

\begin{figure}[t]
    \centering
    \includegraphics[width=\linewidth]{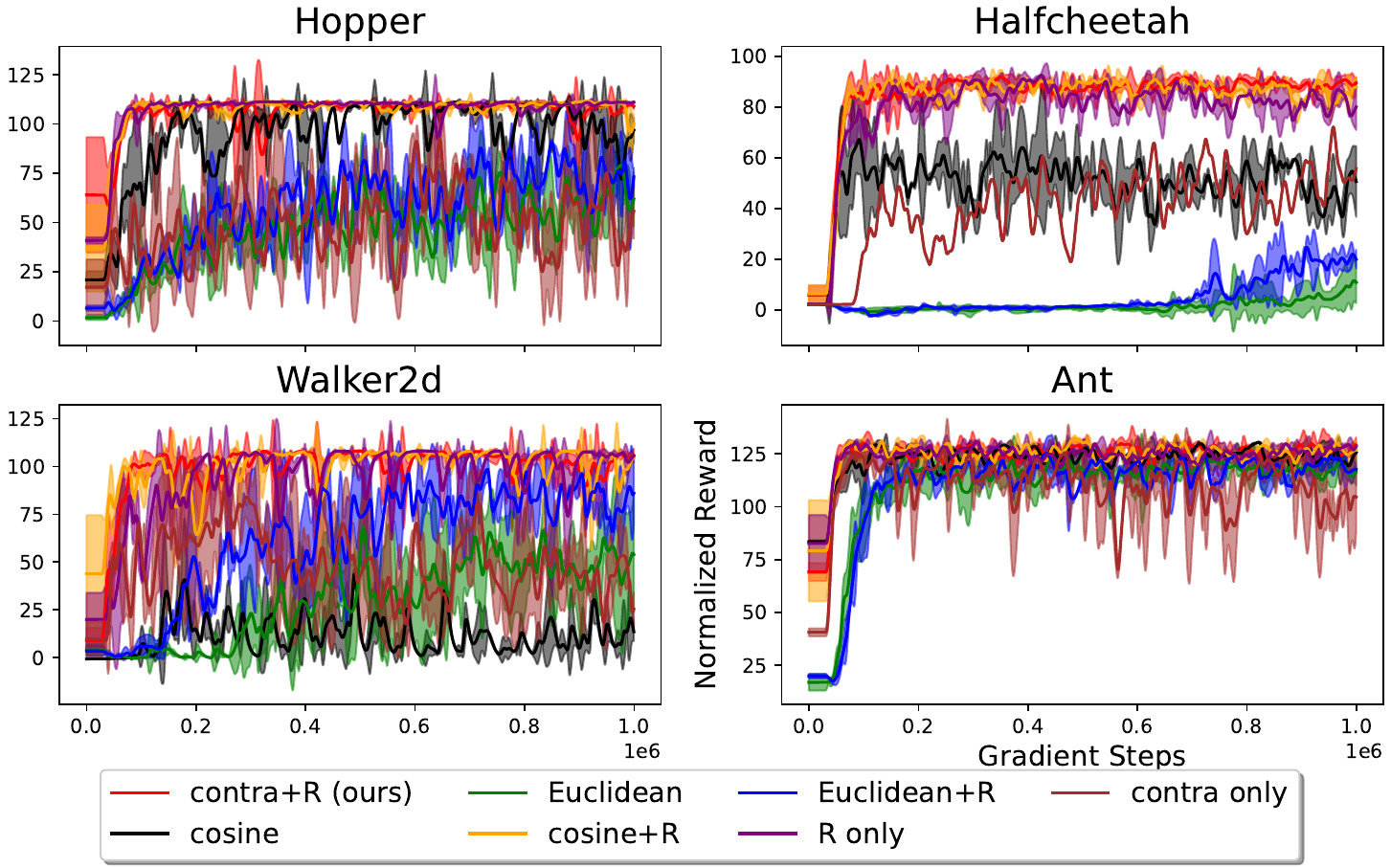}
    \caption{Ablations on the choice of distance metric{\color{black}s}. Our choice of $c(s,s')$, which combines a contrastively learned distance and {\color{black}the discriminator-based component} $R$, performs {\color{black}the} best. {\color{black}The} Euclidean distance fails in our scenario, which further proves the importance of using the primal form instead of the Rubinstein dual form.}
    \label{fig:ablation-dist-MuJoCo}
\end{figure}
{\color{black}
\textbf{Ablations on $\epsilon_1$ and $\epsilon_2$.} In order to show the robustness of PW-DICE to the choice of $\epsilon_1$ and $\epsilon_2$, we conduct an ablation study 
on the MuJoCo environment. Specifically, we test $\epsilon_1\in\{0.1, 0.5, 1\}\times \epsilon_2\in\{0.1, 0.5, 1\}$. The result is shown in Fig.~\ref{fig:ablation-epsilon}. While some choice of hyperparameters leads to failure, PW-DICE is generally robust to the selection of $\epsilon_1$ and $\epsilon_2$. Generally, $\epsilon_1$ should be small to maintain good performance. See more ablations in Appendix~\ref{sec:app-exp-res}.

\begin{figure}[t]
    \centering
    \includegraphics[width=\linewidth]{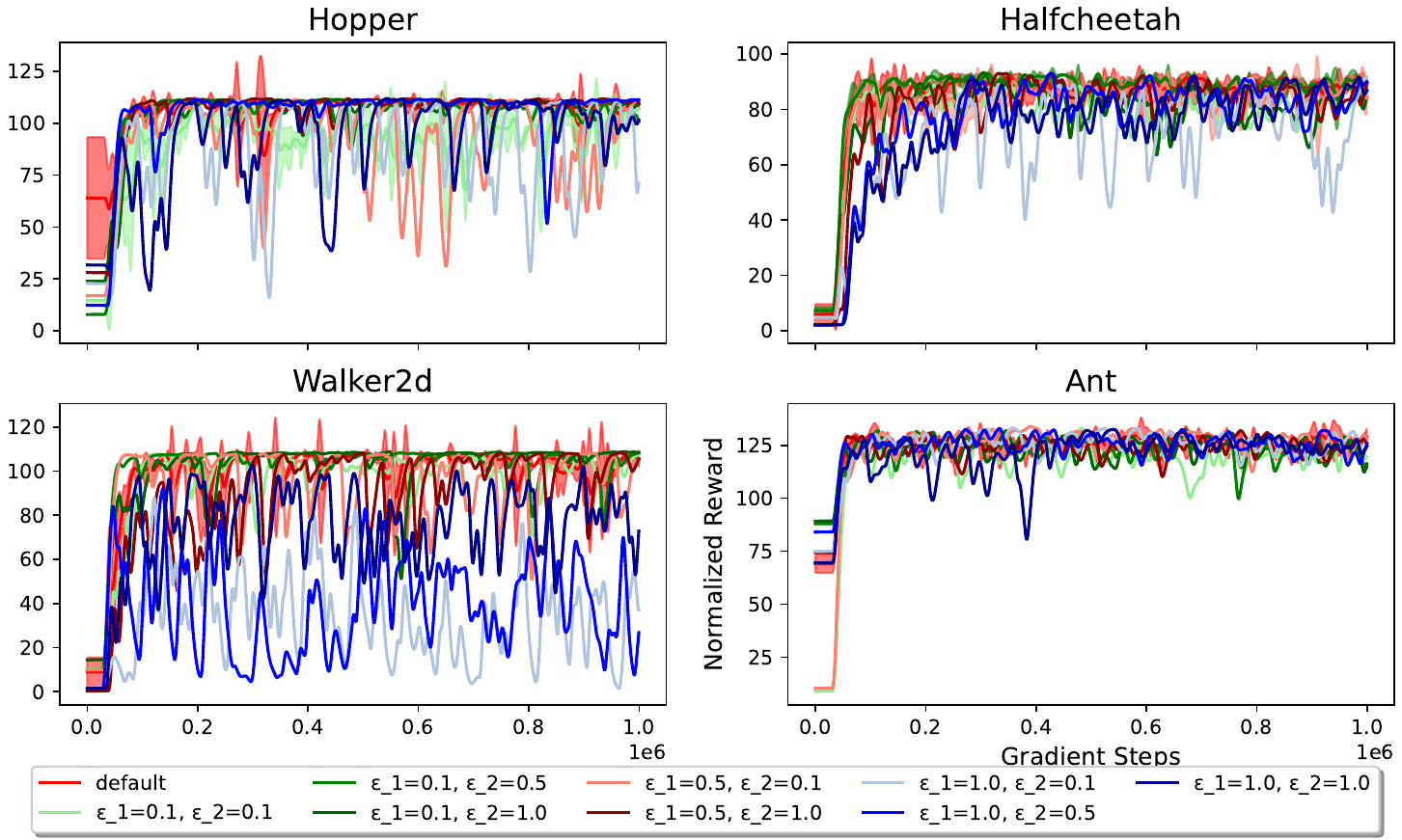}
    \caption{\color{black}Ablation of $\epsilon_1$ and $\epsilon_2$ on the MuJoCo testbed: $\epsilon_1=0.1$ is marked in green, $\epsilon_1=0.5$ is marked in red, and $\epsilon_1=1.0$ is marked in blue. The deeper the color is, the larger $\epsilon_2$ is. Our method is generally robust to hyperparameter changes, though some choice results in failure. Generally, large $\epsilon_1$ leads to worse performance.}
    \label{fig:ablation-epsilon}
\end{figure}

\textbf{Robustness against distorted state representations.} One important motivation for using a learned distance metric is that a fixed distance metric might be limited to the state representation. For example, the Euclidean distance could perform well when learning navigation for a point-mass, where coordinates are given as states. However, the Euclidean distance will no longer be accurate when some of the dimensions undergo scaling (e.g., due to metric changes from inches to meters). While scaling each dimension independently could be alleviated by state normalization, in this experiment we consider a more complicated \textit{distortion} to the state representations.

More specifically, for state $s\in\mathbb{R}^{1\times n}$, we randomly generate a distortion matrix $D=0.1I+D'\in\mathbb{R}^{n\times n}$, where each element of $D'$ is independently and randomly sampled from $\mathcal{N}(0, 4^2)$. The new state exposed to the agent (both in the dataset and evaluation) is calculated as $s'=D's$. We compare our method against SMODICE on several MuJoCo environments. Results are shown in Fig.~\ref{fig:distort}. We observe that our method is generally more robust to poor state representations than SMODICE. 

\begin{figure}[t]
    \centering
    \includegraphics[width=\linewidth]{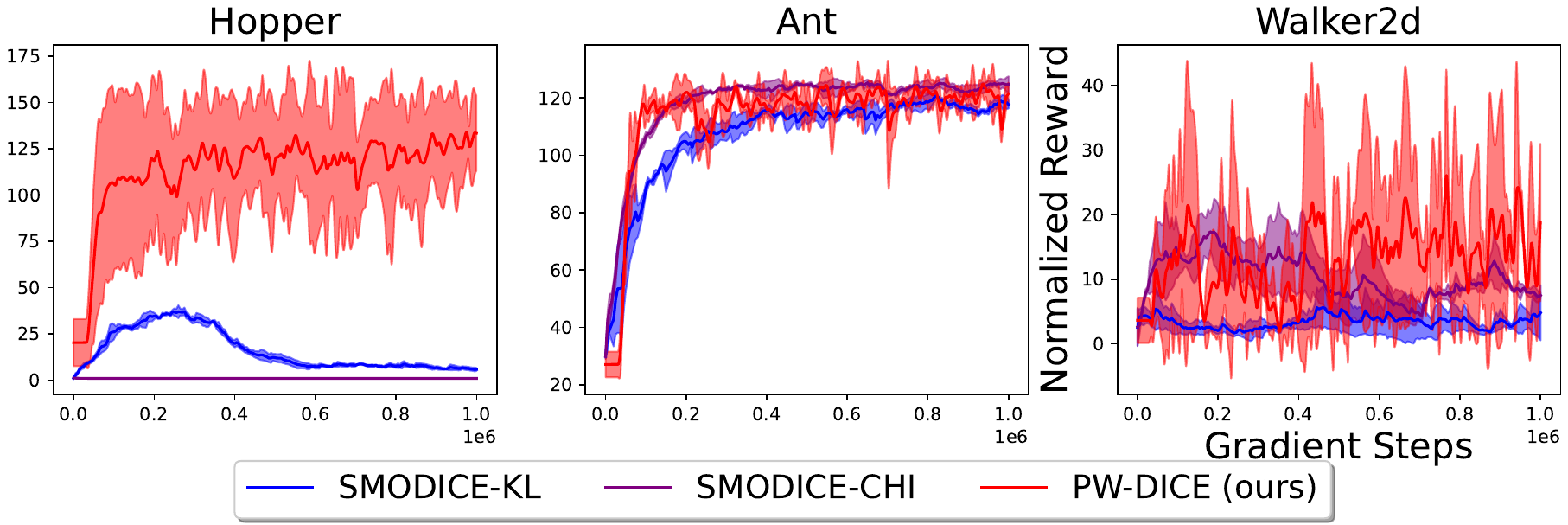}
    \caption{\color{black}Performance comparison between our proposed method, PW-DICE, and SMODICE under distorted state representations. Our method generally outperforms SMODICE.}
    \label{fig:distort}
\end{figure}

} 

\section{Related Work}
\label{sec:relatedwork}
\textbf{Wasserstein Distance for Imitation Learning.} As a metric which is capable of leveraging geometric properties of distributions and which yields gradients for distributions with different support sets, {\color{black}the} Wasserstein distance (also known as \textit{Optimal Transport})~\cite{kantorovich1960mathematical} is a popular choice when studying distribution divergence minimization. 
It is widely used in IL/RL~\cite{Agarwal2021ContrastiveBS,fickinger2022gromov,xiao2019wasserstein,dadashi2020primal,garg2021iq}. Among them, SoftDICE~\cite{sun2021softdice} is the most similar work to our PW-DICE, which also optimizes {\color{black}the} Wasserstein distance under the DICE framework. However, SoftDICE and most Wasserstein-based IL algorithms~\cite{sun2021softdice, xiao2019wasserstein, WDAIL, liu2019state} use Rubinstein-Kantorovich duality~\cite{KR:58,COTFNT}, which limits the underlying distance metric to be Euclidean. There are a few methods optimizing {\color{black}the} primal Wasserstein distance. For example, OTR~\cite{luo2023otr} computes the primal Wasserstein distance between two trajectories and assigns rewards accordingly for offline RL.  PWIL~\cite{dadashi2020primal} uses greedy coupling to simplify the computation of the Wasserstein distance. However, the former struggles in our experiment{\color{black}al} settings, and the latter only optimizes an upper bound of the Wasserstein distance. Moreover, both {\color{black}methods} only use fixed heuristic distance metrics, such as Euclidean~\cite{dadashi2020primal} and cosine~\cite{luo2023otr}. Our PW-DICE {\color{black}addresses} these issues.



\textbf{Offline Imitation Learning from Observation.} Offline Learning from Observation (LfO) aims to learn from expert observations with no labeled action, which is useful in robotics where the expert action is either not available (e.g., {\color{black}in} video{\color{black}s}~\cite{Pari2021TheSE}) or not applicable (e.g., from a different embodiment~\cite{TCN2017}). Three major directions are present in this area: 1) offline planning or RL with assigned, similarity-based reward~\cite{bco, Kumar2019LearningNS}; 2) occupancy divergence minimization, which includes iterative inverse-RL methods~\cite{Zolna2020OfflineLF,Xu2019PositiveUnlabeledRL, Torabi2018GenerativeAI} and DICE~\cite{ma2022smodice,DBLP:conf/iclr/KimSLJHYK22,Kim2022LobsDICEOL,lee2021optidice,Zhu2021OffPolicyIL}; 3) action pseudo-labeling, where the missing actions are predicted with an inverse dynamic model~\cite{TCN2017,chen2019generative,wu2019fisl}. Our PW-DICE falls in the second category but generalizes over SMODICE, unifies $f$-divergence and Wasserstein, and empirically improves upon existing methods.

\textbf{Contrastive Learning for State Representations.} Contrastive learning methods, such as InfoNCE~\cite{infonce} and SIMCLR~\cite{chen2020simple}, aim to find a good representation that satisfies similarity and dissimilarity constraints between particular pairs of data points. Contrastive learning is widely used in reinforcement learning, especially with visual input~\cite{laskin2020curl, Pari2021TheSE, TCN2017} and for meta RL~\cite{Fu2020TowardsEC} to improve the generalizability of agents and mitigate the curse of dimensionality. In these {\color{black}methods}, similarity constraints can come from different augmentations of the same state~\cite{laskin2020curl, Pari2021TheSE}, multiview alignment~\cite{TCN2017}, consistency after reconstruction~\cite{MCRL}, or task context~\cite{Fu2020TowardsEC}. Different from prior work, PW-DICE uses contrastive learning to {\color{black}identify} a good distance metric considering state reachability, while still adopting the reward from the DICE work.


\section{Conclusion}

In this paper, we propose PW-DICE, a DICE method that uses the primal form of the Wasserstein distance and a  contrastively learned distance metric. By adding adequate pessimistic regularizers, we formulate an unconstrained convex optimization and retrieve the policy using weighted behavior cloning. Our method is a generalization of SMODICE, unifying $f$-divergence and Wasserstein minimization in imitation learning. This generalization enables better performance than multiple baselines, such as SMODICE~\cite{ma2022smodice} and LobsDICE~\cite{Kim2022LobsDICEOL}. 

\textbf{Limitations and Future Directions.} In order to obtain an unconstrained optimization formulation, we add KL terms to the objective. This introduces a logsumexp into the final objective. Some {\color{black}studies} argue that logsumexp adds instability to the optimization, {\color{black}due to the use of minibatches: minibatch gradient estimates for logarithms and exponentials of expectations are biased~\cite{sun2021softdice}}. {\color{black}Although we did not} observe this {\color{black}and it has been found to be tolerable in prior work~\cite{ma2022smodice}, this} may be a potential shortcoming for PW-DICE {\color{black}in} even more challenging environments. Thus, one next step is to find a more robust formulation while maintaining the beneficial properties of PW-DICE.

\vspace{-1pt}
\section*{Acknowledgements}

This work was supported in part by NSF under Grants 2008387, 2045586, 2106825, MRI 1725729, NIFA award 2020-67021-32799, the Jump ARCHES endowment through the Health Care Engineering Systems Center at Illinois and the OSF Foundation, and the IBM-Illinois Discovery Accelerator Institute.
\vspace{-1pt}
\section*{Impact Statement}

Our work automates decision-making processes by utilizing expert observations as well as past experience data. While our {\color{black}effort} improves the efficiency of automated task-solving, it could also lead to negative societal impacts in several aspects. For example, since our work is mostly tested on locomotion tasks, there {\color{black}exists} a potential risk {\color{black}of} harmful {\color{black}applications} (e.g., military) of our proposed decision-making techniques. Also, the improvement of automated decision-making {\color{black}may} potentially {\color{black}result in} a reduction of job {\color{black}opportunities}.


\bibliography{example_paper}

\begin{thebibliography}{57}
\providecommand{\natexlab}[1]{#1}
\providecommand{\url}[1]{\texttt{#1}}
\expandafter\ifx\csname urlstyle\endcsname\relax
  \providecommand{\doi}[1]{doi: #1}\else
  \providecommand{\doi}{doi: \begingroup \urlstyle{rm}\Url}\fi

\bibitem[Agarwal et~al.(2021)Agarwal, Machado, Castro, and Bellemare]{Agarwal2021ContrastiveBS}
Agarwal, R., Machado, M.~C., Castro, P.~S., and Bellemare, M.~G.
\newblock Contrastive behavioral similarity embeddings for generalization in reinforcement learning.
\newblock In \emph{ICLR}, 2021.

\bibitem[Agrawal et~al.(2019)Agrawal, Amos, Barratt, Boyd, Diamond, and Kolter]{cvxpy}
Agrawal, A., Amos, B., Barratt, S., Boyd, S., Diamond, S., and Kolter, J.~Z.
\newblock Differentiable convex optimization layers.
\newblock In \emph{NeurIPS}, 2019.

\bibitem[ApS(2019)]{mosek}
ApS, M.
\newblock \emph{The MOSEK optimization toolbox for CVXPY manual.}, 2019.
\newblock URL \url{https://docs.mosek.com/9.0/faq/index.html}.

\bibitem[Boyd \& Vandenberghe(2004)Boyd and Vandenberghe]{boyd2004convex}
Boyd, S. and Vandenberghe, L.
\newblock \emph{Convex optimization}.
\newblock Cambridge University Press, 2004.

\bibitem[Brockman et~al.(2016)Brockman, Cheung, Pettersson, Schneider, Schulman, Tang, and Zaremba]{1606.01540}
Brockman, G., Cheung, V., Pettersson, L., Schneider, J., Schulman, J., Tang, J., and Zaremba, W.
\newblock {OpenAI} gym, 2016.

\bibitem[Chang et~al.(2020)Chang, Gupta, and Gupta]{youtube-navigation}
Chang, M., Gupta, A., and Gupta, S.
\newblock Semantic visual navigation by watching {YouTube} videos.
\newblock In \emph{NeurIPS}, 2020.

\bibitem[Chen et~al.(2021{\natexlab{a}})Chen, Nair, and Finn]{Chen2021LearningGR}
Chen, A.~S., Nair, S., and Finn, C.
\newblock Learning generalizable robotic reward functions from ``in-the-wild'' human videos.
\newblock In \emph{RSS}, 2021{\natexlab{a}}.

\bibitem[Chen et~al.(2021{\natexlab{b}})Chen, Lu, Rajeswaran, Lee, Grover, Laskin, Abbeel, Srinivas, and Mordatch]{chen2021decisiontransformer}
Chen, L., Lu, K., Rajeswaran, A., Lee, K., Grover, A., Laskin, M., Abbeel, P., Srinivas, A., and Mordatch, I.
\newblock Decision transformer: Reinforcement learning via sequence modeling.
\newblock In \emph{NeurIPS}, 2021{\natexlab{b}}.

\bibitem[Chen et~al.(2020)Chen, Kornblith, Norouzi, and Hinton]{chen2020simple}
Chen, T., Kornblith, S., Norouzi, M., and Hinton, G.
\newblock A simple framework for contrastive learning of visual representations.
\newblock In \emph{NeurIPS}, 2020.

\bibitem[Chen et~al.(2019)Chen, Li, Li, Jiang, Qi, and Song]{chen2019generative}
Chen, X., Li, S., Li, H., Jiang, S., Qi, Y., and Song, L.
\newblock Generative adversarial user model for reinforcement learning based recommendation system.
\newblock In \emph{ICML}, 2019.

\bibitem[Dadashi et~al.(2021)Dadashi, Hussenot, Geist, and Pietquin]{dadashi2020primal}
Dadashi, R., Hussenot, L., Geist, M., and Pietquin, O.
\newblock Primal wasserstein imitation learning.
\newblock In \emph{ICLR}, 2021.

\bibitem[Dai et~al.(2017)Dai, He, Pan, Boots, and Song]{pmlr-v54-dai17a}
Dai, B., He, N., Pan, Y., Boots, B., and Song, L.
\newblock {Learning from conditional distributions via dual embeddings}.
\newblock In \emph{AISTATS}, 2017.

\bibitem[Fickinger et~al.(2022)Fickinger, Cohen, Russell, and Amos]{fickinger2022gromov}
Fickinger, A., Cohen, S., Russell, S., and Amos, B.
\newblock Cross-domain imitation learning via optimal transport.
\newblock In \emph{ICLR}, 2022.

\bibitem[Fu et~al.(2020{\natexlab{a}})Fu, Tang, Hao, Chen, Feng, Li, and Liu]{Fu2020TowardsEC}
Fu, H., Tang, H., Hao, J., Chen, C., Feng, X., Li, D., and Liu, W.
\newblock Towards effective context for meta-reinforcement learning: {An} approach based on contrastive learning.
\newblock In \emph{AAAI}, 2020{\natexlab{a}}.

\bibitem[Fu et~al.(2020{\natexlab{b}})Fu, Kumar, Nachum, Tucker, and Levine]{fu2020d4rl}
Fu, J., Kumar, A., Nachum, O., Tucker, G., and Levine, S.
\newblock {D4RL}: Datasets for deep data-driven reinforcement learning.
\newblock \emph{ArXiv:2004.07219}, 2020{\natexlab{b}}.

\bibitem[Garg et~al.(2021)Garg, Chakraborty, Cundy, Song, and Ermon]{garg2021iq}
Garg, D., Chakraborty, S., Cundy, C., Song, J., and Ermon, S.
\newblock {IQ-Learn}: Inverse soft-{Q} learning for imitation.
\newblock In \emph{NeurIPS}, 2021.

\bibitem[Ghasemipour et~al.(2019)Ghasemipour, Zemel, and Gu]{Seyed2019Divergence}
Ghasemipour, S., Zemel, R., and Gu, S.
\newblock A divergence minimization perspective on imitation learning methods.
\newblock In \emph{CoRL}, 2019.

\bibitem[{Gurobi Optimization, LLC}(2023)]{gurobi}
{Gurobi Optimization, LLC}.
\newblock Gurobi optimizer reference manual, 2023.
\newblock URL \url{https://www.gurobi.com}.

\bibitem[Hakhamaneshi et~al.(2022)Hakhamaneshi, Zhao, Zhan, Abbeel, and Laskin]{Hakhamaneshi2022FIST}
Hakhamaneshi, K., Zhao, R., Zhan, A., Abbeel, P., and Laskin, M.
\newblock Hierarchical few-shot imitation with skill transition models.
\newblock In \emph{ICLR}, 2022.

\bibitem[Ho \& Ermon(2016)Ho and Ermon]{NIPS2016_cc7e2b87}
Ho, J. and Ermon, S.
\newblock Generative adversarial imitation learning.
\newblock In \emph{NIPS}, 2016.

\bibitem[Jin et~al.(2021)Jin, Yang, and Wang]{jin2021pessimism}
Jin, Y., Yang, Z., and Wang, Z.
\newblock Is pessimism provably efficient for offline {RL}?
\newblock In \emph{ICML}, 2021.

\bibitem[Kantorovich \& Rubinstein(1958)Kantorovich and Rubinstein]{KR:58}
Kantorovich, L. and Rubinstein, G.~S.
\newblock On a space of totally additive functions.
\newblock \emph{Vestnik Leningradskogo Universiteta}, 1958.

\bibitem[Kantorovich(1960)]{kantorovich1960mathematical}
Kantorovich, L.~V.
\newblock Mathematical methods of organizing and planning production.
\newblock \emph{Management science}, 1960.

\bibitem[Kim et~al.(2022{\natexlab{a}})Kim, Seo, Lee, Jeon, Hwang, Yang, and Kim]{DBLP:conf/iclr/KimSLJHYK22}
Kim, G., Seo, S., Lee, J., Jeon, W., Hwang, H., Yang, H., and Kim, K.
\newblock {DemoDICE}: Offline imitation learning with supplementary imperfect demonstrations.
\newblock In \emph{ICLR}, 2022{\natexlab{a}}.

\bibitem[Kim et~al.(2022{\natexlab{b}})Kim, Lee, Jang, Yang, and Kim]{Kim2022LobsDICEOL}
Kim, G.-H., Lee, J., Jang, Y., Yang, H., and Kim, K.
\newblock {LobsDICE}: Offline learning from observation via stationary distribution correction estimation.
\newblock In \emph{NeurIPS}, 2022{\natexlab{b}}.

\bibitem[Kiran et~al.(2021)Kiran, Sobh, Talpaert, Mannion, Sallab, Yogamani, and Perez]{autonomous-driving}
Kiran, B., Sobh, I., Talpaert, V., Mannion, P., Sallab, A., Yogamani, S., and Perez, P.
\newblock Deep reinforcement learning for autonomous driving: A survey.
\newblock \emph{IEEE Transactions on Intelligent Transportation Systems}, 2021.

\bibitem[Kostrikov et~al.(2019)Kostrikov, Agrawal, Dwibedi, Levine, and Tompson]{DBLP:conf/iclr/KostrikovADLT19}
Kostrikov, I., Agrawal, K.~K., Dwibedi, D., Levine, S., and Tompson, J.
\newblock Discriminator-actor-critic: Addressing sample inefficiency and reward bias in adversarial imitation learning.
\newblock In \emph{ICLR}, 2019.

\bibitem[Kostrikov et~al.(2020)Kostrikov, Nachum, and Tompson]{valuedice}
Kostrikov, I., Nachum, O., and Tompson, J.
\newblock Imitation learning via off-policy distribution matching.
\newblock In \emph{ICLR}, 2020.

\bibitem[Kostrikov et~al.(2022{\natexlab{a}})Kostrikov, Nair, and Levine]{CQL}
Kostrikov, I., Nair, A., and Levine, S.
\newblock Conservative {Q}-learning for offline reinforcement learning.
\newblock In \emph{ICLR}, 2022{\natexlab{a}}.

\bibitem[Kostrikov et~al.(2022{\natexlab{b}})Kostrikov, Nair, and Levine]{kostrikov2021iql}
Kostrikov, I., Nair, A., and Levine, S.
\newblock Offline reinforcement learning with implicit {Q}-learning.
\newblock In \emph{ICLR}, 2022{\natexlab{b}}.

\bibitem[Kumar et~al.(2019)Kumar, Gupta, and Malik]{Kumar2019LearningNS}
Kumar, A., Gupta, S., and Malik, J.
\newblock Learning navigation subroutines from egocentric videos.
\newblock In \emph{CoRL}, 2019.

\bibitem[Laskin et~al.(2020)Laskin, Srinivas, and Abbeel]{laskin2020curl}
Laskin, M., Srinivas, A., and Abbeel, P.
\newblock {CURL}: Contrastive unsupervised representations for reinforcement learning.
\newblock In \emph{ICML}, 2020.

\bibitem[Lee et~al.(2021)Lee, Jeon, Lee, Pineau, and Kim]{lee2021optidice}
Lee, J., Jeon, W., Lee, B.-J., Pineau, J., and Kim, K.-E.
\newblock {OptiDICE}: Offline policy optimization via stationary distribution correction estimation.
\newblock In \emph{ICML}, 2021.

\bibitem[Liu et~al.(2020)Liu, Ling, Mu, and Su]{liu2019state}
Liu, F., Ling, Z., Mu, T., and Su, H.
\newblock State alignment-based imitation learning.
\newblock In \emph{ICLR}, 2020.

\bibitem[Luo et~al.(2023)Luo, Jiang, Cohen, Grefenstette, and Deisenroth]{luo2023otr}
Luo, Y., Jiang, Z., Cohen, S., Grefenstette, E., and Deisenroth, M.~P.
\newblock Optimal transport for offline imitation learning.
\newblock In \emph{ICLR}, 2023.

\bibitem[Ma et~al.(2022)Ma, Shen, Jayaraman, and Bastani]{ma2022smodice}
Ma, Y.~J., Shen, A., Jayaraman, D., and Bastani, O.
\newblock Smodice: Versatile offline imitation learning via state occupancy matching.
\newblock In \emph{ICML}, 2022.

\bibitem[Nachum et~al.(2019)Nachum, Chow, Dai, and Li]{DBLP:conf/nips/NachumCD019}
Nachum, O., Chow, Y., Dai, B., and Li, L.
\newblock {DualDICE}: Behavior-agnostic estimation of discounted stationary distribution corrections.
\newblock In \emph{NeurIPS}, 2019.

\bibitem[Pari et~al.(2022)Pari, Shafiullah, Arunachalam, and Pinto]{Pari2021TheSE}
Pari, J., Shafiullah, N. M.~M., Arunachalam, S.~P., and Pinto, L.
\newblock The surprising effectiveness of representation learning for visual imitation.
\newblock In \emph{RSS}, 2022.

\bibitem[Peyr\'e \& Cuturi(2019)Peyr\'e and Cuturi]{COTFNT}
Peyr\'e, G. and Cuturi, M.
\newblock Computational optimal transport.
\newblock \emph{Foundations and Trends in Machine Learning}, 2019.

\bibitem[Polyanskiy(2020)]{MITnote}
Polyanskiy, Y.
\newblock $f$-divergences, 2020.
\newblock URL \url{https://people.lids.mit.edu/yp/homepage/data/LN_fdiv.pdf}.

\bibitem[Sermanet et~al.(2017)Sermanet, Lynch, Hsu, and Levine]{TCN2017}
Sermanet, P., Lynch, C., Hsu, J., and Levine, S.
\newblock Time-contrastive networks: Self-supervised learning from multi-view observation.
\newblock In \emph{CVPRW}, 2017.

\bibitem[Sikchi et~al.(2024)Sikchi, Zhang, and Niekum]{Sikchi2023ImitationFA}
Sikchi, H.~S., Zhang, A., and Niekum, S.
\newblock Imitation from arbitrary experience: A dual unification of reinforcement and imitation learning methods.
\newblock In \emph{ICLR}, 2024.

\bibitem[Stanczuk et~al.(2021)Stanczuk, Etmann, Kreusser, and Schonlieb]{WGANWF}
Stanczuk, J., Etmann, C., Kreusser, L., and Schonlieb, C.-B.
\newblock Wasserstein {GANs} work because they fail (to approximate the wasserstein distance).
\newblock \emph{ArXiv:2103.01678}, 2021.

\bibitem[Sun et~al.(2021)Sun, Mahajan, Hofmann, and Whiteson]{sun2021softdice}
Sun, M., Mahajan, A., Hofmann, K., and Whiteson, S.
\newblock {SoftDICE} for imitation learning: Rethinking off-policy distribution matching.
\newblock \emph{ArXiv:2106.03155}, 2021.

\bibitem[Torabi et~al.(2018)Torabi, Warnell, and Stone]{bco}
Torabi, F., Warnell, G., and Stone, P.
\newblock Behavioral cloning from observation.
\newblock In \emph{IJCAI}, 2018.

\bibitem[Torabi et~al.(2019)Torabi, Warnell, and Stone]{Torabi2018GenerativeAI}
Torabi, F., Warnell, G., and Stone, P.
\newblock Generative adversarial imitation from observation.
\newblock In \emph{ICML Workshop on Imitation, Intent, and Interaction}, 2019.

\bibitem[Van~der Maaten \& Hinton(2008)Van~der Maaten and Hinton]{van2008visualizing}
Van~der Maaten, L. and Hinton, G.
\newblock Visualizing data using t-{SNE}.
\newblock \emph{JMLR}, 2008.

\bibitem[Wan et~al.(2021)Wan, Zhang, Xiong, and Ye]{infonce}
Wan, C., Zhang, T., Xiong, Z., and Ye, H.
\newblock Representation learning for fault diagnosis with contrastive predictive coding.
\newblock In \emph{CAA Symposium on Fault Detection, Supervision, and Safety for Technical Processes (SAFEPROCESS)}, 2021.

\bibitem[Wu et~al.(2019)Wu, Piergiovanni, and Ryoo]{wu2019fisl}
Wu, A., Piergiovanni, A., and Ryoo, M.~S.
\newblock Model-based behavioral cloning with future image similarity learning.
\newblock In \emph{CoRL}, 2019.

\bibitem[Xiao et~al.(2019)Xiao, Herman, Wagner, Ziesche, Etesami, and Linh]{xiao2019wasserstein}
Xiao, H., Herman, M., Wagner, J., Ziesche, S., Etesami, J., and Linh, T.~H.
\newblock Wasserstein adversarial imitation learning.
\newblock \emph{arXiv preprint arXiv:1906.08113}, 2019.

\bibitem[Xu \& Denil(2019)Xu and Denil]{Xu2019PositiveUnlabeledRL}
Xu, D. and Denil, M.
\newblock Positive-unlabeled reward learning.
\newblock In \emph{CoRL}, 2019.

\bibitem[Xu et~al.(2022)Xu, Zhan, Yin, and Qin]{DWBC}
Xu, H., Zhan, X., Yin, H., and Qin, H.
\newblock Discriminator-weighted offline imitation learning from suboptimal demonstrations.
\newblock In \emph{NeurIPS}, 2022.

\bibitem[Yu et~al.(2019)Yu, Quillen, He, Julian, Hausman, Finn, and Levine]{yu2019meta}
Yu, T., Quillen, D., He, Z., Julian, R., Hausman, K., Finn, C., and Levine, S.
\newblock Meta-{World}: A benchmark and evaluation for multi-task and meta reinforcement learning.
\newblock In \emph{CoRL}, 2019.

\bibitem[Zhang et~al.(2020)Zhang, Wang, Ma, Xia, Yang, Li, and Li]{WDAIL}
Zhang, M., Wang, Y., Ma, X., Xia, L., Yang, J., Li, Z., and Li, X.
\newblock Wasserstein distance guided adversarial imitation learning with reward shape exploration.
\newblock In \emph{Data Driven Control and Learning Systems (DDCLS)}, 2020.

\bibitem[Zhu et~al.(2023)Zhu, Xia, Wu, Deng, Zhou, Qin, Liu, and Li]{MCRL}
Zhu, J., Xia, Y., Wu, L., Deng, J., Zhou, W., Qin, T., Liu, T.-Y., and Li, H.
\newblock Masked contrastive representation learning for reinforcement learning.
\newblock \emph{IEEE TPAMI}, 2023.

\bibitem[Zhu et~al.(2020)Zhu, Lin, Dai, and Zhou]{Zhu2021OffPolicyIL}
Zhu, Z., Lin, K., Dai, B., and Zhou, J.
\newblock Off-policy imitation learning from observations.
\newblock In \emph{NeurIPS}, 2020.

\bibitem[Zolna et~al.(2020)Zolna, Novikov, Konyushkova, Gulcehre, Wang, Aytar, Denil, de~Freitas, and Reed]{Zolna2020OfflineLF}
Zolna, K., Novikov, A., Konyushkova, K., Gulcehre, C., Wang, Z., Aytar, Y., Denil, M., de~Freitas, N., and Reed, S.~E.
\newblock Offline learning from demonstrations and unlabeled experience.
\newblock In \emph{NeurIPS Workshop on Offline Reinforcement Learning}, 2020.

\end{thebibliography}
\bibliographystyle{icml2024}

\newpage
\appendix
\onecolumn
\newpage 
\appendix

\section*{Appendix: Offline Imitation from Observation via Primal Wasserstein State Occupancy Matching}

The appendix is organized as follows. We first present the pseudo-code of PW-DICE in Sec.~\ref{sec:pseudocode}. Then, we
rigorously introduce the most important mathematical concepts of our work in Sec.~\ref{sec:mathcon}, which include state, state-action, and state-pair occupancy, as well as $f$-divergences and Fenchel conjugate. After that, in Sec.~\ref{sec:proof}, we provide detailed  derivations omitted in the main paper, as well as the corresponding proofs. In Sec.~\ref{sec:expdetail}, we provide a detailed description of our experiments. In Sec.~\ref{sec:app-exp-res}, we provide additional experimental results, including auxiliary metrics,  experimental results using an identical softmax expert (which is actually suboptimal, see Sec.~\ref{sec:softmax}) as LobsDICE~\cite{Kim2022LobsDICEOL} in the tabular experiment, and ablations in MuJoCo environments. In Sec.~\ref{sec:notlist}, we summarize our notation. Finally, in Sec.~\ref{sec:resource}, we list the computational resource that we use during the training process.

\section{Pseudo-code}
\label{sec:pseudocode}

\begin{algorithm}[tb]
   \caption{PW-DICE}
   \label{alg:code}
\begin{algorithmic}
   \STATE {\bfseries Input:} Expert state-only dataset $E$, task-agnostic state-action dataset $I$, distance metric $c:|S|\times|S|\rightarrow\mathbb{R}$
   \STATE {\bfseries Input:} Triple head network for dual variable $\lambda_{\theta_1}(s, h):|S|\times\{0,1,2\}\rightarrow\mathbb{R}$ parameterized by $\theta_1$, actor $\pi_{\theta_2}(a|s)$ parameterized by $\theta_2$, the number of epoches $N$, batch size $B$, learning rate $\eta$.
   \STATE Initialize initial state dataset $I_{\text{ini}}=\{\}$.
   \FOR{$\tau\in I$}
   \STATE Add first state of $\tau$ to $I_{\text{ini}}$ 
   \ENDFOR
   \FOR{$\text{epoch}=1$ {\bfseries to} $N$}
   \FOR{$(s,a,s')\in I$}
   \STATE Sample state $s^I$ from $I$, state $s^E$ from $E$ 
   \STATE $l_1\leftarrow \frac{1}{\epsilon_1}(\lambda_{\theta_1}(s^I, 1)+\lambda_{\theta_1}
   (s^E, 2)-c(s^I,s^E))-\log B$
   \STATE $l_2\leftarrow \frac{1}{\epsilon_2}(-\gamma\lambda_{\theta_1}(s', 0)+\lambda_{\theta_1}(s, 0)-\lambda_{\theta_1}(s, 1))-\log B$
   \STATE Sample initial states $s_{\text{ini}}$ from $I_{\text{ini}}$
   \STATE $l_3\leftarrow -\frac{1}{B}[(1-\gamma)\lambda_{\theta_1}(s_{\text{ini}}, 0)+\lambda_{\theta_1}(s^E, 2)]$
   \STATE $l\leftarrow\text{logsumexp}(l_1)+\text{logsumexp}(l_2)+l_3$
   \STATE $\theta_1\leftarrow\theta_1-\eta\frac{\partial l}{\partial \theta_1}$
   \ENDFOR
   \FOR{$(s,a,s')\in I$}
   \STATE $v\leftarrow \frac{1}{\epsilon_2}(-\gamma\lambda_{\theta_1}(s', 0)+\lambda_{\theta_1}(s, 0)-\lambda_{\theta_1}(s, 1))$
   \STATE $l\leftarrow\exp(v)\log\pi_{\theta_2}(a|s)$
   \STATE $\theta_2\leftarrow\theta_2-\eta\frac{\partial l}{\partial \theta_2}$
   \ENDFOR
   \ENDFOR
\end{algorithmic}
\end{algorithm}

Alg.~\ref{alg:code} details the training process of our main algorithm. Upon implementation, we normalize coefficient $v$ over the whole task-agnostic dataset $I$ for better stability. See Sec.~\ref{sec:mjc} for the implementation detail of contrastive learning for the distance metric.

\section{Mathematical Concepts}
\label{sec:mathcon}
In this section, we introduce three important  concepts used in the paper, which are state/state-action/state-pair occupancy, $f$-divergence, and Fenchel conjugate. The first one is the key concept used throughout this work, the second is used in our motivations, 
and the last is used in Sec.~\ref{sec:optdetail}.

\subsection{State, State-Action, and State-Pair Occupancy} Consider an MDP $(S, A, T, r, \gamma)$ with initial state distribution $p_0$ and infinite horizon; at the $t$-th timestep, we denote the current state as $s_t$ and the action as $a_t$. Then, with a fixed policy $\pi$, the probability of $\Pr(s_t=s)$ and $\Pr(a_t=a)$ for any $s,a$ are determined. Based on this, the \textit{state occupancy}, which is the state visitation frequency under policy $\pi$, is defined as $d^\pi_s(s)=(1-\gamma)\sum_{t=0}^\infty\gamma^t\Pr(s_t=s)$. Similarly, we define the  \textit{state-action occupancy} as $d^\pi_{sa}(s,a)=(1-\gamma)\sum_{t=0}^\infty \gamma^t \Pr(s_t=s, a_t=a)$. Some work such as LobsDICE also uses \textit{state-pair occupancy}, which is defined as $d^\pi_{ss}(s,s')=(1-\gamma)\sum_{t=0}^\infty\gamma^t 
 \Pr(s_t=s,s_{t+1}=s')$. In this work, we denote the average policy that generates the task-agnostic dataset $I$ as $\pi^I$ with state occupancy $d^I_s$ and state-action occupancy $d^I_{sa}$, and the expert policy that generates the expert dataset $E$ as $\pi^E$ with state occupancy $d^E_s$.

\subsection{$f$-divergence} The $f$-divergence is a measure of distance between probability distributions $p, q$ and is widely used in the machine learning community~\cite{Seyed2019Divergence}. For two probability distributions $p, q$ on domain $\mathcal{X}$ based on any continuous and convex function $f$, the $f$-divergence between $p$ and $q$ is defined as

\begin{equation}
    D_f(p\|q)=\mathbb{E}_{x\sim q}\left[f\left(\frac{p(x)}{q(x)}\right)\right].
\end{equation}

For instance, when $f(x)=x\log x$, we have $D_f(p\|q)=\mathbb{E}_{x\sim q}\frac{p(x)}{q(x)}\log\frac{p(x)}{q(x)}=\mathbb{E}_{x\sim p}\log\frac{p(x)}{q(x)}$, which induces the KL-divergence. When $f(x)=(x-1)^2$, we have $D_f(p\|q)=\mathbb{E}_{x\sim q}((\frac{p(x)-q(x)}{q(x)})^2)$, which induces the $\chi^2$-divergence.

\subsection{Fenchel Conjugate} Fenchel conjugate is widely used in DICE methods for either debiasing estimations~\cite{DBLP:conf/nips/NachumCD019} or solving formulations with stronger constraints to get numerically more stable objectives~\cite{ma2022smodice}. PW-DICE uses the Fenchel conjugate for the latter. For a vector space $\Omega$ and a convex, differentiable function $f:\Omega\rightarrow\mathbb{R}$, the Fenchel conjugate of $f(x)$ is defined as

\begin{equation}
    f_*(y)=\max_{x\in\Omega}\langle x, y\rangle-f(x),
\end{equation}

where $\langle\cdot,\cdot\rangle$ is the inner product over $\Omega$.

\section{Mathematical Derivations}
\label{sec:proof}

In this section, we provide the detailed  derivations omitted in the main paper due to the page limit. In Sec.~\ref{sec:smodice}, we briefly introduce SMODICE to clarify the motivation of using Wasserstein distance and how SMODICE is related to our proposed PW-DICE. In Sec.~\ref{sec:optdetail}, we provide a detailed derivation of our objective, omitted in Sec.~\ref{sec:reg}. Sec.~\ref{sec:component} and Sec.~\ref{sec:invariance} are complements to Sec.~\ref{sec:optdetail}: in Sec.~\ref{sec:component}, we provide a detailed derivation on the elements of condensed representation of constraints, $A$ and $b$, and in Sec.~\ref{sec:invariance}, we explain why additional constraints are applied during one step of the derivation process while the optimal solution remains the same. Finally, in Sec.~\ref{sec:thm2}, we provide a detailed proof for the claim that our method is a generalization of SMODICE in Sec.~\ref{sec:reg}.

\subsection{SMODICE}
\label{sec:smodice}

SMODICE~\cite{ma2022smodice} is a state-of-the-art offline LfO method. It minimizes the $f$-divergence between the state occupancy of the learner's policy $\pi$ and the expert policy $\pi^E$, i.e., the objective is

\begin{equation}
    \min_\pi D_f(d^\pi_s(s)\|d^E_s(s)), \text{s.t.}\ \pi \text{ is feasible.}
\end{equation}

Here, the feasibility of $\pi$ is the same as the Bellman flow constraint (the second row of constraints in Eq.~\eqref{eq:primal_main}) in the main paper. To take the only information source of environment dynamics, which is the task-agnostic dataset $I$, into account, the objective is relaxed to

\begin{equation}
\label{eq:smodiceprimal}
\max_\pi \mathbb{E}_{s\sim d^\pi}\log\frac{d^E_s(s)}{d^I_s(s)}-D_f(d^\pi_{sa}(s,a)\|d^I_{sa}(s,a)), \text{s.t. $\pi$ is a feasible policy.}
\end{equation}

Here, $D_f$ can be any divergence that is not smaller than KL-divergence (SMODICE mainly studies $\chi^2$-divergence). The first term, $\log\frac{d^E_s(s)}{d^I_s(s)}$ indicates the relative importance of the state. The more often the expert visits a particular state $s$ than non-expert policies, the more possible that $s$ is a desirable state. Reliance on such a ratio introduces a theoretical limitation: the assumption that $d^I_s(s)>0$ wherever $d^E_s(s)>0$ must be made, which does not necessarily hold in a high-dimensional space. Thus, we introduce a hyperpamater of $\alpha$ to mix the distribution in the denominator in our reward design.

By transforming the constrained problem into an unconstrained problem in the Lagrange dual space, SMODICE optimizes the following objective (assuming the use of KL-divergence):

\begin{equation}
\label{eq:smodicedual}
    \min_V (1-\gamma)\mathbb{E}_{s\sim p_0} [V(s)]+\log \mathbb{E}_{(s,a,s')\sim I}\exp\left[\log\frac{d^E_s(s)}{d^I_s(s)}+\gamma V(s')-V(s)\right],
\end{equation}
where $p_0$ is the initial state distribution and $\gamma$ is the discount factor. As stated in Sec.~\ref{sec:reg}, this objective is a special case of PW-DICE with $c(s,s')=\log\frac{d^E_s(s)}{d^I_s(s)}$, $\epsilon_2=1, \epsilon_1\rightarrow 0$. LobsDICE~\cite{Kim2022LobsDICEOL} is similar in spirit; however, it minimizes the state-pair divergence $\text{KL}(d^\pi_{ss}\|d^E_{ss})$ instead. 

\subsection{Detailed Derivation of Our  Objective}
\label{sec:optdetail}

As mentioned in Eq.~\eqref{eq:primal_main} in Sec.~\ref{sec:reg}, our primal objective is 

\begin{equation}
\begin{aligned}
&\min_{\Pi, \pi}\sum_{s_i\in S}\sum_{s_j\in S}\Pi(s_i, s_j)c(s_i,s_j), \text{s.t. }  d^\pi_{sa}\geq 0, \Pi\geq 0;\\ 
&\forall s\in S, d^\pi_s(s)=(1-\gamma)p_0(s)+\gamma\sum_{\bar{s},\bar{a}}d^\pi_{sa}(\bar{s},\bar{a})p(s|\bar{s},\bar{a});\\
&\forall s_i, s_j\in S,\sum_k\Pi(s_k, s_j)=d^E_s(s_j), \sum_k\Pi(s_i,s_k)=d^\pi_s(s_i).
\label{eq:primal_main_app}
\end{aligned}
\end{equation}

Before smoothing our objective, for readability, we rewrite our main objective in Eq.~\eqref{eq:primal_main_app} as an LP problem over a single vector $x=\begin{bmatrix}\Pi\\d^\pi_{sa}\end{bmatrix}\in\mathbb{R}^{|S|\times(|S|+|A|)}$, where $\Pi\in\mathbb{R}^{|S|\times|S|}$ and $d^\pi_{sa}\in\mathbb{R}^{|S|\times|A|}$ are flattened in a row-first manner. Correspondingly, we extend the distance $c$ between states to $c': (|S|(|S|+|A|))\times (|S|(|S|+|A|))\rightarrow\mathbb{R}$, such that $c'=c$ on the original domain of $c$ and $c'=0$ otherwise. Further, we summarize all linear equality constraints in $Ax=b$.  Eq.~\eqref{eq:primal_main} is then equivalent to
\begin{equation}
    \min_{x\geq 0} (c')^Tx \quad\text{s.t.}\quad Ax=b, x\geq 0.
\label{eq:simplify}
\end{equation}

It is easy to see that the Lagrange dual form of Eq.~\eqref{eq:simplify} is also a constrained optimization. 
In order to convert the optimization to an unconstrained one, we modify the objective as follows (same as Eq.~\eqref{eq:reg} in the main paper):
\begin{equation}
\begin{aligned}
\min_x (c')^Tx&+\epsilon_1 D_f(\Pi\|U)+\epsilon_2 D_f(d^\pi_{sa}\|d^I_{sa}),\\ &\text{s.t. } Ax=b, x\geq 0,
\label{eq:reg_app}
\end{aligned}
\end{equation}
where $U(s,s')=d^E_s(s)d^I_s(s')$, i.e., $U$ is the product of two independent distributions $d^E_s$ and $d^I_s$. $\epsilon_1>0, \epsilon_2>0$ are hyperparameters, and $D_f$ can be any $f$-divergence. Note, although an $f$-divergence is used, unlike SMODICE~\cite{ma2022smodice} or LobsDICE~\cite{Kim2022LobsDICEOL}, such formulation does not require data coverage of the task-agnostic data over the expert data. The two regularizers are ``pessimistic,'' which encourages the agents to stay within the support set of the dataset. This has been used in offline IL/RL~\cite{jin2021pessimism}. 

With the regularized objective in Eq.~\eqref{eq:reg}, we now consider its Lagrange dual problem:
\begin{equation}
\begin{aligned}
\label{eq:lag1}
\max_\lambda\min_{x\geq 0}L(\lambda, x),\text{ where }
L(\lambda, x)=(c')^Tx+\epsilon_1D_f(\Pi\|U)+\epsilon_2D_f(d^\pi_{sa}\|d^I_{sa})-\lambda^T(Ax-b).
\end{aligned}
\end{equation}


While Eq.~\eqref{eq:lag1} only has the non-negativity constraint, its domain is the non-negative numbers. Thus the objective can be optimized as being unconstrained. To obtain a practical and stable solution, a single-level optimization is preferred. To do so, one could consider using the KKT condition~\cite{boyd2004convex}, and set the derivative of the inner-level optimization to $0$. However, such an approach will lead to an exp function in the objective~\cite{MITnote, Kim2022LobsDICEOL}, which is numerically unstable~\cite{Kim2022LobsDICEOL}. To avoid this, we first rewrite Eq.~\eqref{eq:lag1} \textit{with negated $L(\lambda, x)$} to separate $\Pi$ and $d^\pi_{sa}$ in $x$:

\begin{equation}
\begin{aligned}
\label{eq:separate}
\min_\lambda&\max_{x\geq 0}-L(\lambda, x)\\
\min_\lambda&\left\{\epsilon_1\max_{\Pi\in\Delta(S^2)}\left[\frac{(A_1^T\lambda-c)}{\epsilon_1}^T\Pi-D_f(\pi\|U)\right]+\epsilon_2\max_{d^\pi_{sa}\in\Delta(S\cdot A)}\left[\frac{(A_2^T\lambda)}{\epsilon_2}^Td^\pi_{sa}-D_f(d^\pi_{sa}\|d^I_{sa})\right]-b^T\lambda\right\}.    
\end{aligned}
\end{equation}

In Eq.~\eqref{eq:separate}, we have $A=\begin{bmatrix}A_1\\A_2\end{bmatrix}$, where $A_1\in\mathbb{R}^{(|S|\times|S|)\times M}$, $A_2\in\mathbb{R}^{(|S|\times|A|)\times M}$, and $M=3|S|$ is the number of equality constraints in the primal form. See Sec.~\ref{sec:component} for elements in $A, A_1, A_2$, and $b$. Two points are worth noting in Eq.~\eqref{eq:separate}.

First, we append two extra constraints, which are $\Pi\in\Delta, d^\pi_{sa}\in\Delta$. These constraints do not affect the final result for the following fact:

\begin{lemma}
\label{thm:lemma}
For any MDP and feasible expert policy $\pi^E$, the inequality constraints in Eq.~\eqref{eq:primal_main} with $\Pi\geq 0, d^\pi_{sa}\geq 0$ and $\Pi\in\Delta, d^\pi_{sa}\in\Delta$ are equivalent.
\end{lemma}

The detailed proof of Lemma~\ref{thm:lemma} is given in  Appendix~\ref{sec:invariance}. In a word, the optimal solution of Eq.~\eqref{eq:lag1}, as long as it satisfies all constraints in the primal form, must have $\Pi\in\Delta, d^\pi_{sa}\in\Delta$. 

Second, we decompose the max operator into two independent maximizations, as the equality constraints that correlate $\Pi$ and $d^\pi_{sa}$ are all relaxed in the dual. 
 With Eq.~\eqref{eq:separate}, we now apply the following theorem from  SMODICE~\cite{ma2022smodice}:

\begin{theorem} 
\label{thm:smodice}
With mild assumptions~\cite{pmlr-v54-dai17a}, for any $f$-divergence $D_f$, probability distribution $p,q$ on domain $\mathcal{X}$ and function $y: \mathcal{X}\rightarrow\mathbb{R}$, we have

\begin{equation}
\max_{p\in\Delta(\mathcal{X})}\mathbb{E}_{x\sim p}[y(x)]-D_f(p\|q)=\mathbb{E}_{x\sim q}[f_*(y(x))].
\end{equation}

For maximizer $p^*(x)=\mathop{\arg\max}_{p\in\Delta(\mathcal{X})}\mathbb{E}_{x\sim q}[f_*(y(x))]$, we have $p^*(x)=q(x)f'_*(y(x))$, where $f_*(\cdot)$ is the Fenchel conjugate of $f$, and $f'_*$ is its derivative.
\end{theorem}




A complete proof can be found at theorem $7.14^*$ in~\citet{MITnote}. The rigorous definition of $f$-divergence and Fenchel conjugate are in Appendix~\ref{sec:mathcon}. For this work, we mainly consider KL-divergence as $D_f$, which corresponds to $f(x)=x\log x$, and $f_*(x)=\text{logsumexp}(x)$ to be the Fenchel dual function with $x\in\Delta$~\cite{boyd2004convex}\footnote{$\chi^2$-divergence does not work as well as KL-divergence in MuJoCo environments. See Appendix~\ref{sec:chidiv} for details.}. With Thm.~\ref{thm:smodice}, we set $p=\Pi, x=\lambda, y(x)=\frac{A_1^T\lambda-c}{\epsilon_1}$ for the first max operator, and set $p=d^\pi_{sa}, x=\lambda, y(x)=\frac{A_2^T\lambda}{\epsilon_2}$ for the second max operator. Then, we get the following single-level convex objective:

\begin{equation}
\begin{aligned}
\label{eq:singlelevel}
\min_\lambda \epsilon_1 \log\mathbb{E}_{s_i\sim I, s_j\sim E}\exp\left(\frac{(A_1^T\lambda-c)^T}{\epsilon_1}\right)+\epsilon_2\log\mathbb{E}_{(s_i,a_j)\sim I}\exp\left(\frac{A^T_2\lambda}{\epsilon_2}\right)-b^T\lambda,
\end{aligned}
\end{equation}

with which, by considering the components of $A_1, A_2$ and $b$ in Appendix~\ref{sec:component}, we have our final objective stated in Eq.~\eqref{eq:final} and maximizer stated in Sec.~\ref{sec:reg}.

\subsection{Components of $A,b$ in Eq.~\eqref{eq:simplify}}
\label{sec:component}

In this subsection, we discuss in detail the entries of $A$, $b$ in Eq.~\eqref{eq:simplify}. In Eq.~\eqref{eq:simplify}, we summarize all equality constraints in Eq.~\eqref{eq:primal_main} as $Ax=b$, $x=\begin{bmatrix}\Pi\\d^\pi_{sa}\end{bmatrix}$, where $\Pi, d^\pi_{sa}$ are flattened in a row-first manner. Thus, we have $x_{:i|S|+j}=\Pi(s_i, s_j)$, and $x_{|S|^2+i|A|+j}=d^\pi_{sa}(s_i, a_j)$. 

We further assume that in $A$ and $b$, the first $|S|$ rows are the Bellman flow constraints
\begin{equation}
\forall s, \sum_{a}d^\pi_{sa}(s,a)-\gamma \sum_{\bar{s},\bar{a}}p(s|\bar{s},\bar{a})d^\pi_{sa}(\bar{s},\bar{a})=(1-\gamma)p_0(s).
\label{eq:con1}
\end{equation}
The second $|S|$ rows are the $\sum_j\Pi(s_i,s_j)=d^\pi_s(s_i)$ marginal constraints

\begin{equation}
\label{eq:con2}
\forall s, \sum_{s'}\Pi(s, s')=\sum_{a}d^\pi_{sa}(s,a).
\end{equation}

The third $|S|$ rows are the $\sum_i\Pi(s_i,s_j)=d^E_s(s_j)$ constraints
\begin{equation}
\forall s, \sum_{s'}\Pi(s',s)=\sum_{a}d^E_{sa}(s,a).
\label{eq:con3}
\end{equation}

Thus, 
we have $A_{i, |S|^2+j|A|+k}=-\gamma p(s_i|s_j,a_k)$ for $i\in\{1,2,\dots,|S|\}$, $A_{i, |S|^2+i|A|:|S|^2+(i+1)|A|}=1$ for $i\in\{1,2,\dots,|S|\}$ (Eq.~\eqref{eq:con1}), $A_{i+|S|,i|S|+j}=1$ for $i\in\{1,2,\dots,|S|\}$, $A_{i+|S|,|S|^2+i|A|+j}=-1$ (Eq.~\eqref{eq:con2}), and $A_{i+2|S|, j|S|+i}=1$ (Eq.~\eqref{eq:con3}). Other entries of $A$ are $0$. $A_1$ in Eq.~\eqref{eq:separate} are the first $|S|\times|S|$ rows of $A$, and $A_2$ are the last $|S|\times|A|$ rows of $A$.

For vector $b$, we have
\begin{equation}
b=\begin{bmatrix}(1-\gamma)p_0\\0\\d^E_{s}
\end{bmatrix}.
\end{equation}

\setcounter{lemma}{0}
\subsection{Lemma~\ref{thm:lemmaapp}}
\label{sec:invariance}

In this section, we provide a proof of Lemma~\ref{thm:lemmaapp} used in Appendix~\ref{sec:optdetail}. The Lemma reads as follows:

\begin{lemma}
\label{thm:lemmaapp}
For any MDP and feasible expert policy $\pi^E$, the inequality constraints in Eq.~\eqref{eq:primal_main} with $\Pi\geq 0, d^\pi_{sa}\geq 0$ and $\Pi\in\Delta, d^\pi_{sa}\in\Delta$ are equivalent.
\end{lemma}

\begin{proof}
According to the equality constraint, $\sum_{s}\Pi(s,s')=d^E_s(s')$ for any $s'$. Thus, we have $\sum_{s'}\sum_{s}\Pi(s,s')=\sum_{s'}d^E_s(s')=1$ by the definition of state occupancy. Thus $\Pi\geq 0$ is equivalent to $\Pi\geq \Delta$. Similarly, by summing over both sides of the Bellman flow equality constraint, we have
\begin{equation}
\begin{aligned}
\sum_{s}d^\pi_s(s)&=\sum_{s}(1-\gamma)p_0(s)+\sum_{s}\gamma\sum_{\bar{s},\bar{a}}d^\pi_{sa}(\bar{s},\bar{a})p(s|\bar{s},\bar{a})\\
\sum_{s,a}d^\pi_{sa}(s,a)&=(1-\gamma)+\gamma\sum_{s}\sum_{\bar{s},\bar{a}}d^\pi_{sa}(\bar{s},\bar{a})p(s|\bar{s},\bar{a})\\
\sum_{s,a}d^\pi_{sa}(s,a)&=(1-\gamma)+\gamma\sum_{s'}\sum_{s,a}d^\pi_{sa}(s,a)p(s'|s,a)\\
\sum_{s,a}d^\pi_{sa}(s,a)(1-\gamma\sum_{s'}p(s'|s,a))&=1-\gamma\\
\sum_{s,a}d^\pi_{sa}(s,a)&=1
\end{aligned}
\end{equation}

given that $p_0$ and the transition function are legal. Thus, $d^\pi_{sa}\geq 0$ is equivalent to $d^\pi_{sa}\in\Delta$.

\end{proof}

Intuitively, by adding the extra constraints, we can assume that redundant equality constraints exist in Eq.~\eqref{eq:primal_main}, and they are not relaxed in the Lagrange dual. By imposing more strict constraints over the dual form, the Fenchel conjugate yields a numerically more stable formulation.

\subsection{Proof of Generalization over SMODICE}
\label{sec:thm2}

In Sec.~\ref{sec:reg}, we claim that our proposed method, PW-DICE, is a generalization over SMODICE. More specifically, we have the following claim and corollary:

\begin{claim}
\label{thm:generalapp}
If $c(s_i,s_j)=-\log\frac{d^E_s(s_i)}{d^I_s(s_i)}, \epsilon_2=1$, then as $\epsilon_1\rightarrow 0$, Eq.~\eqref{eq:final} is equivalent to the SMODICE objective with KL divergence.  
\end{claim}

\begin{corollary}
\label{thm:corolapp}
If $c(s_i,s_j)=-\log\frac{d^E_s(s_i)}{d^I_s(s_i)}, \epsilon_2=1$, then as $\epsilon_1\rightarrow 0$, Eq.~\eqref{eq:separate} is equivalent to SMODICE with any $f$-divergence.
\end{corollary}

We first provide a simple proof from the primal perspective:

\begin{proof} (Primal Perspective)
According to Eq.~\eqref{eq:smodiceprimal} and Eq.~\eqref{eq:primal_main}, the SMODICE and PW-DICE primal objectives are as follows:

\begin{equation}
\begin{aligned}
\min_x (c')^Tx+\epsilon_1 D_f(\Pi\|U)+\epsilon_2 D_f(d^\pi_{sa}\|d^I_{sa}), \text{s.t. } Ax=b, x\geq 0;\text{ (PW-DICE)}\\
\max_\pi \mathbb{E}_{s\sim d^\pi}\log\frac{d^E_s(s)}{d^I_s(s)}-D_f(d^\pi_{sa}(s,a)\|d^I_{sa}(s,a)), \text{s.t. $\pi$ is a feasible policy. (SMODICE)}\\
\label{eq:compare}
\end{aligned}
\end{equation}

Here, $x=\begin{bmatrix}d^\pi_s\\\Pi\end{bmatrix}$. Note: 
1) $Ax=b, x\geq 0$ contains three equality constraints: the Bellman flow equation (which is the same as ``$\pi$ is a feasible policy''), $\sum_{s'}\Pi(s,s')=d^\pi_s(s),$ and $\sum_{s}\Pi(s,s')=d^E(s')$; 
2) $(c')^Tx=\sum_{s,s'}c(s,s')\Pi(s,s')$. 
Thus, we have
\begin{equation}
\sum_{s}\sum_{s'}c(s,s')\Pi(s,s')=\sum_{s}\log\frac{d^E_s(s)}{d^I_s(s)}\sum_{s'}\Pi(s,s')=-\mathbb{E}_{s\sim d^\pi_s}\log\frac{d^E_s(s)}{d^I_s(s)}.
\end{equation}

Therefore, when $\epsilon_1=0, \epsilon_2=1$, $c'(s,s')=c(s,s')=-\log \frac{d^E_s(s)}{d^I_s(s)}$, the objective of PW-DICE and SMODICE are negated version of each other (with one maximizing and the other minimizing), and the constraints on $d^\pi_{sa}$ are identical. Since $\Pi$ is also solvable (one apparent solution is $\Pi=d^\pi_s\otimes d^E_s$), the two objectives are identical, and thus the objectives in Eq.~\eqref{eq:compare} are equivalent. Both the Claim and the Corollary are proved, since we do not specify $D_f$.
\end{proof}

However, this claim is unintuitive in its dual form: as we always have $\epsilon_1>0, \epsilon_2>0$ in the dual form, the behavior of  $\lim_{\epsilon_1\rightarrow 0}\epsilon_1\log \mathbb{E}_{s_i\sim I, s_j\sim E}\exp\left(\frac{\lambda_{i+|S|}+\lambda_{j+2|S|}-c(s_i,s_j)}{\epsilon_1}\right)$ in Eq.~\eqref{eq:final} is non-trivial. Thus, here we give another proof for Claim~\ref{thm:generalapp} directly from the dual perspective for KL-divergence as $D_f$ in the continuous space:

\begin{proof} (Dual Perspective, KL-divergence, continuous space)
First, we prove by contradiction that 
\begin{equation}
    \lim_{\epsilon_1\rightarrow 0}\epsilon_1\log \mathbb{E}_{s\sim I,s'\sim E}\exp{\left(\frac{\lambda_{s+|S|}+\lambda_{s'+2|S|}-c(s,s')}{\epsilon_1}\right)}
\end{equation}
is not the $\max$ operator, because at the optimum we have $\lambda_{s+|S|}+\lambda_{s'+2|S|}-c(s,s')$ to be equal for every $d^I_s(s)>0, d^E_s(s')>0$. Otherwise, assume the state pair $(s,s')$ has the largest $\lambda_{s+|S|}+\lambda_{s'+2|S|}-c(s_0,s'_0)$; because $\epsilon_1$ can be arbitrarily close to $0$, there exists $\epsilon_1$ small enough such that there exists $s\neq s_0$ or $s'\neq s'_0$ that makes the infinitesimal increment of $\lambda_s$ or $\lambda_s'$ worthy (i.e., partial derivative with respect to $\lambda_s$ or $\lambda_s'$ greater than $0$).

Then, we have 
\begin{equation}
\label{eq:lim}
\begin{aligned}
    &\lim_{\epsilon_1\rightarrow 0}\epsilon_1\log \mathbb{E}_{s\sim I,s'\sim E}\exp{\left(\frac{\lambda_{s+|S|}+\lambda_{s'+2|S|}-c(s,s')}{\epsilon_1}\right)}\\=&\mathbb{E}_{s\sim I,s'\sim E}\left(\lambda_{s+|S|}+\lambda_{s'+2|S|}-c(s,s')\right)\\=&\mathbb{E}_{s\sim I}\left[\lambda_{s+|S|}+\log\frac{d^E_s(s)}{d^I_s(s)}\right]+\mathbb{E}_{s'\sim E}\lambda_{s'+2|S|}.
\end{aligned}
\end{equation}

Note that $\lambda_{s'+2|S|}$ in Eq.~\eqref{eq:lim} is cancelled out with the linear term $-\mathbb{E}_{s'\sim E}\lambda_{s'+2|S|}$ in the objective (see Eq.~\eqref{eq:final}) later, so the value of $\lambda_{s'+2|S|}$ does not matter anymore. That means, for any $\lambda_{s'+2|S|}$, there exists an optimal solution. Therefore, without loss of generality, we let $\lambda_{s'+2|S|}=0$. The objective then becomes 

\begin{equation}
\begin{aligned}
&\epsilon_1\log \mathbb{E}_{s\sim I}\exp\left(\frac{\lambda_{s+|S|}+\log\frac{d^E_s(s)}{d^I_s(s)}}{\epsilon_1}\right)+\\&\epsilon_2\log \mathbb{E}_{(s,a,s')\sim I}\exp{\left(\frac{-\gamma \lambda_{s'}+\lambda_s-\lambda_{s+|S|}}{\epsilon_2}\right)}-(1-\gamma)\mathbb{E}_{s\sim p_0}\lambda_s.
\end{aligned}
\end{equation}

Then, we can use the same trick on $\epsilon_1\rightarrow 0$ and infer that
$\lambda_{s+|S|}=-\log\frac{d^E_s(s)}{d^I_s(s)}+Q$, where $Q$ is some constant. Then, we have our optimization objective to be

\begin{equation}
L(\lambda)=Q+\epsilon_2\log \mathbb{E}_{(s,a,s')\sim I}\exp{\left(\frac{-\gamma \lambda_{s'}+\lambda_s+\log\frac{d^E_s(s)}{d^I_s(s)}-Q}{\epsilon_2}\right)}-(1-\gamma)\mathbb{E}_{s\sim p_0}\lambda_s.
\end{equation}

Note that $Q$ is cancelled out again, which means that the value of $Q$ does not matter. Without loss of generality, we set $Q=0$, and then we obtain the SMODICE objective with KL-divergence.
\end{proof}


\section{Experimental Details}
\label{sec:expdetail}
\subsection{Tabular MDP} 

\textbf{Experimental Settings.} We adopt the tabular MDP experiment from LobsDICE~\cite{Kim2022LobsDICEOL}. For the tabular experiment, there are $20$ states in the MDP and $4$ actions for each state $s$; each action $a$ leads to four uniformly chosen states $s'_1, s'_2, s'_3, s'_4$. The vector of probability distribution over the four following states is determined by the formula $(p(s'_1|s,a), p(s'_2|s,a)), p(s'_3|s,a), p(s'_4|s,a)=(1-\eta)X+\eta Y$, where $X\sim\text{Categorical}(\frac{1}{4}, \frac{1}{4},\frac{1}{4},\frac{1}{4})$, and $Y\sim\text{Dirichlet}(1, 1, 1, 1)$. $\eta\in[0, 1]$ controls the randomness of the transition: {\color{black}$\eta=0$} means deterministic, and $\eta=1$ means highly stochastic. The agent always starts from state $s_0$, and can only get a reward of $+1$ by reaching a particular state $s_{x}$. $x$ is chosen such that the optimal value function  $V^*(s_0)$ is minimized. The discount factor $\gamma$ is set to $0.95$. 

\textbf{Dataset Settings.} For each MDP, the expert dataset is generated using a deterministic optimal policy with infinite horizon, and the task-agnostic dataset is generated similarly but with a uniform policy. Note that we use a different expert policy from the softmax policy of LobsDICE, because we found the value function for each state to be quite close due to the high connectivity of the MDP. Thus, the ``expert'' softmax policy is actually near-uniform and severely suboptimal.

\textbf{Selection of Hyperparameters.} There is no hyperparameter selection for SMODICE. For LobsDICE, we follow the settings in their paper, which is $\alpha=0.1$. For our method, we use $\epsilon_1=\epsilon_2=0.01$ for the version with regularizer, and $\epsilon_1=\epsilon_2=0$ for the version with Linear Programming (LP).

\subsection{MuJoCo Environment}
\label{sec:mjc}
\textbf{Experimental Settings.} Following SMODICE~\cite{ma2022smodice}, we test four widely adopted MuJoCo locomotion environments: hopper, halfcheetah, ant, and walker2d, and two more challenging locomotion environments, antmaze and kitchen. Below is the detailed description for each environment.  See Fig.~\ref{fig:env-illu} for an illustration.

\begin{enumerate}
    \item \textbf{Hopper.} Hopper is a 2D environment where the agent controls a single-legged robot to jump forward. The state is $11$-dimensional, which includes the angle and velocity for each joint of the robot; the action is $3$-dimensional, each of which controls the torque applied on a particular joint.
    \item \textbf{Halfcheetah}. In Halfcheetah, the agent controls a cheetah-like robot to run forward. Similar to Hopper, the environment is also 2D, with $17$-dimensional state space describing the coordinate and velocity and $6$-dimensional action space controlling torques on its joints.
    \item \textbf{Ant.} Ant is a 3D environment where the agent controls a quadrupedal robotic ant to move forward. The $111$-dimensional state space includes the coordinate and velocity of each joint. The action space is $8$-dimensional.
    \item \textbf{Walker2d.} Walker2d, as its name suggests, is a 2D environment where the agent controls a two-legged robot to walk forward. The state space is $27$-dimensional and the action space is $8$-dimensional.  

    \item\textbf{Antmaze}. In this work, we consider the U-maze task, where the agent needs to manipulate a $8$-DoF robotic ant with $29$-dimensional state to crawl from one end of the maze to another, as illustrated in Fig.~\ref{fig:env-illu}.
    
    
    \item\textbf{Kitchen.} Franka Kitchen in D4RL is a challenging environment, where the agent manipulates a 9-DoF robotic arm and tries to complete $4$ sequential subtasks. Subtask candidates include moving the kettle, opening the microwave, turning on the bottom or top burner, opening the left or right cabinet, and turning on the lights. The state space describes the status of the robot and the goal and current location of the target items. The state is $60$-dimensional. 
    

\end{enumerate}

\begin{figure}
    \centering
    \begin{minipage}[c]{0.3\linewidth}
\subfigure[Hopper]{\includegraphics[width=\linewidth]{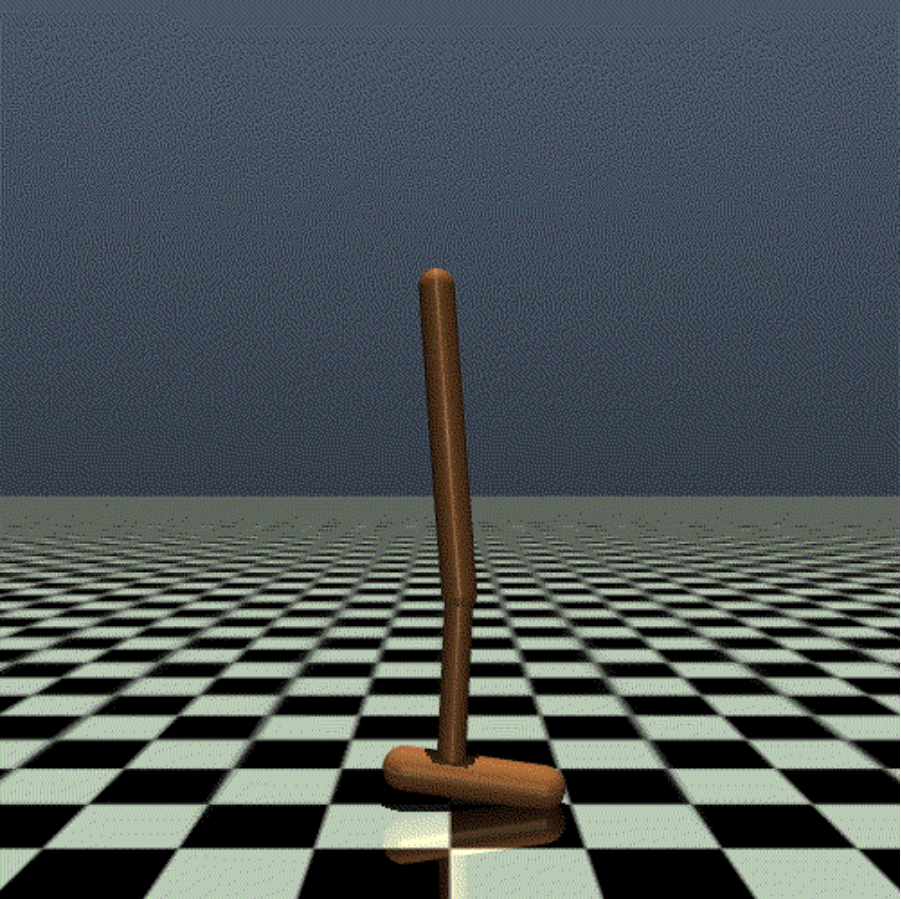}}
\subfigure[Halfcheetah]{\includegraphics[width=\linewidth]{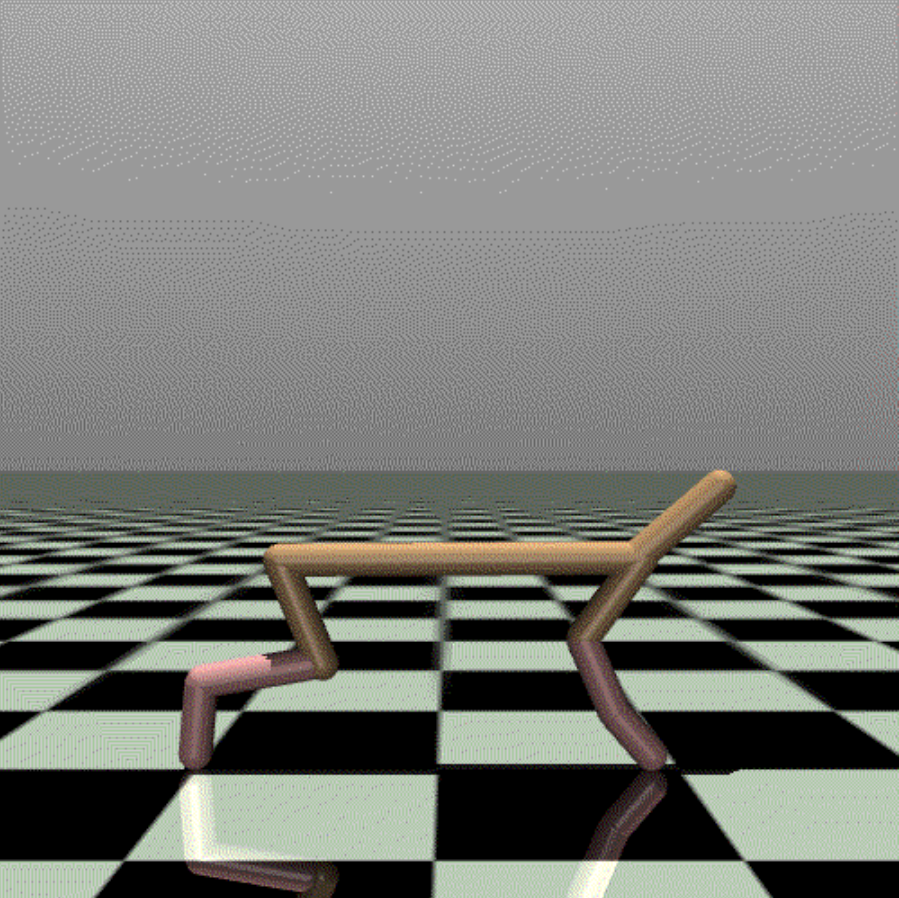}}
\end{minipage}
\begin{minipage}[c]{0.3\linewidth}
\subfigure[Ant]{\includegraphics[width=\linewidth]{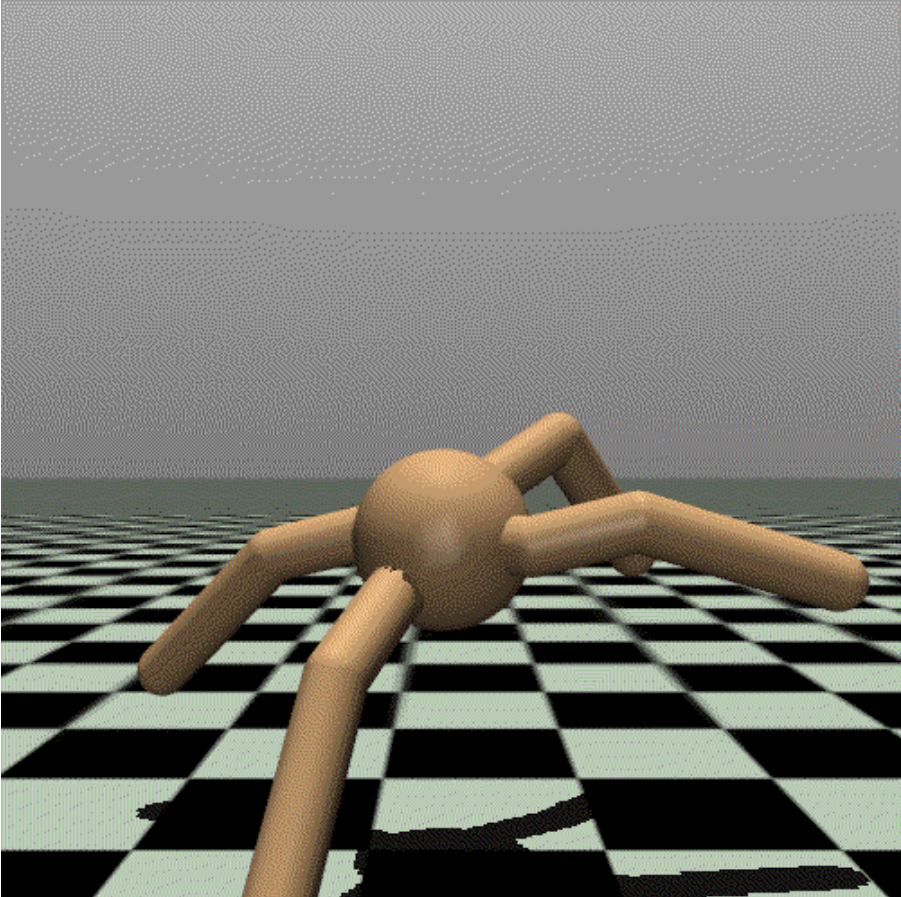}}
\subfigure[Walker2d]{\includegraphics[width=\linewidth]{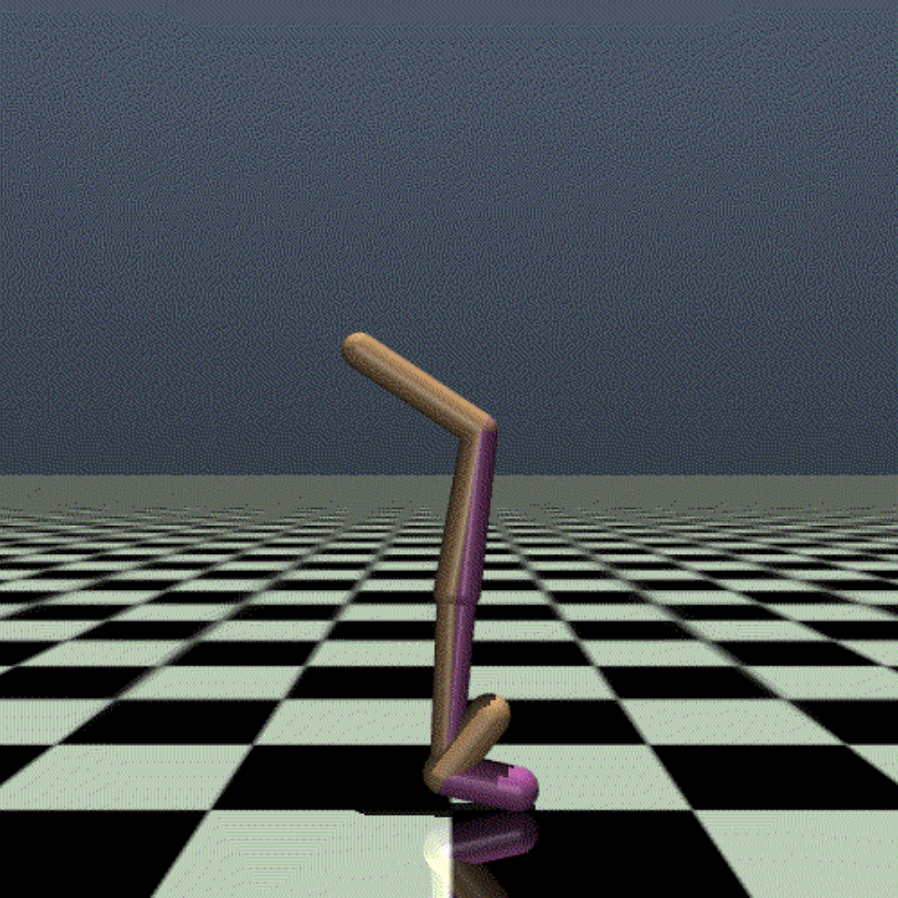}}
\end{minipage}
\begin{minipage}[c]{0.3\linewidth}
\subfigure[Antmaze]{\includegraphics[width=\linewidth]{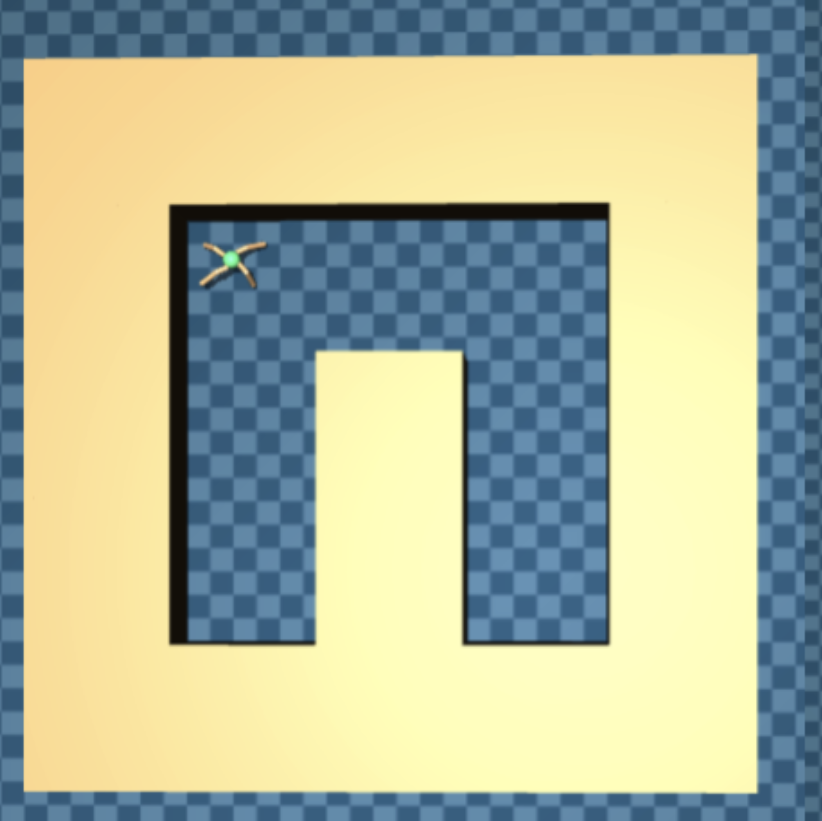}}
\subfigure[Kitchen]{\includegraphics[width=\linewidth]{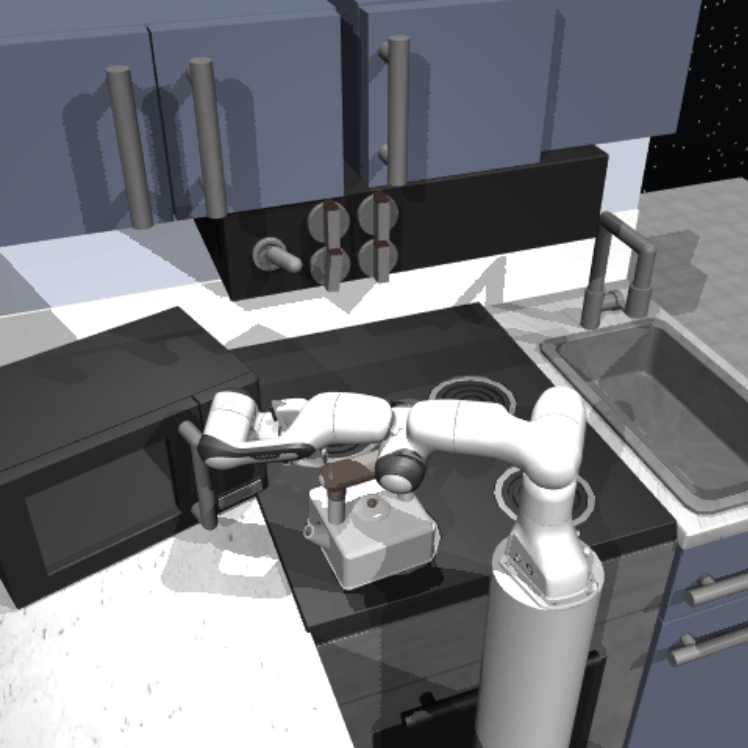}}
\end{minipage}

\caption{Illustration of environments tested in Sec.~\ref{sec:MuJoCoexp} based on OpenAI Gym~\cite{1606.01540} and D4RL~\cite{fu2020d4rl}.}
    \label{fig:env-illu}
\end{figure}

\textbf{Dataset Settings.} We adopt the same settings as SMODICE~\cite{ma2022smodice}. SMODICE uses a single trajectory ($1000$ states) from the ``expert-v2'' dataset in D4RL~\cite{fu2020d4rl} as the expert dataset $E$. For the task-agnostic dataset $I$, SMODICE uses the concatenation of $200$ trajectories ($200K$ state-action pairs) from ``expert-v2'' and  the whole ``random-v2'' dataset ($1M$ state-action pairs).

\textbf{Selection of Hyperparameters.} Tab.~\ref{tab:uniquehpp} summarizes our hyperparameters, which are also the hyperpameters of plain Behavior Cloning if applicable. For baselines (SMODICE, LobsDICE, ORIL, OTR, and DWBC), we use the hyperparameters reported in their paper (unless the hyperparameter values in the paper and the code differ, in which  case we report the values from the code).

\textbf{Implementation of Contrastively Learned Distance.} We build our contrastive learning module on the implementation of FIST~\cite{Hakhamaneshi2022FIST}. The contrastive learning embedding framework is trained for $200$ epochs over the task-agnostic dataset. In each epoch, we use a batch size of $4096$, where each data point consists of a consecutive state pair $(s_i, s'_i)$, $i\in\{1,2,\dots,4096\}$, and the states are $n$-dimensional. An encoding layer $g(s):\mathbb{R}^{n}\rightarrow\mathbb{R}^{M\times 1}$ with $M=32$ embedding dimensions is first applied on each $s_i$ and $s'_i$. Then, the score matrix $L\in\mathbb{R}^{4096\times 4096}$ is calculated, such that $L_{i,j}=g(s_i)^TWg(s'_j)$. $W\in\mathbb{R}^{M\times M}=\text{softplus}(W_0)\text{softplus}(W_0^T)$ is a learnable matrix, and it is structured for better stability. Finally, each row of $L$ is viewed as the score for a $4096$-way classification. The label is the identity matrix. Such training paradigm gives us the  loss also provided in Eq.~\eqref{eq:infonce}:

\begin{equation}
\label{eq:infonce-app}
L_c=\log\frac{\exp\left(q^TWk_+\right)}{\exp\left(q^TWk_+\right)+\sum_{k_-}\exp\left(q^TWk_-\right)}.
\end{equation}
Here, $q=g(s_i)$ is the query (anchor), $W$ is the weight matrix, $k_+=g(s'_i)$ is a positive key, and $k_-\in\{g(s'_j)|j\neq i\}$ are negative keys. This objective essentially amounts to a $4096$-way classification task, where for the $i$-th sample the correct label is $i$. 

\begin{table}[t]
    \centering
    \scriptsize
    \begin{tabular}{cccc}
        \hline
          Type & Hyperparameter & Value & Note  \\
        \hline
        Disc. & Network Size & [256, 256] \\
              & Activation Function & Tanh \\
              & Learning Rate & 0.0003 \\
              & Training Length & 40K steps \\
              & Batch Size & 512 \\ 
              & Optimizer & Adam \\ 
        Actor & Network Size & [256, 256] \\
              & Activation Function & ReLU \\
              & Learning Rate & 0.001 \\
              & Weight Decay & $10^{-5}$ \\
              & Training length & 1M steps \\
              & Batch Size & 1024 \\
              & Optimizer & Adam \\
              & Tanh-Squashed & Yes \\
        Critic & Network Size & [256, 256] \\
               & Activation Function & ReLU \\
               & Learning Rate & 0.0003 \\
               & Training Length & 1M steps \\
               & Batch Size & 1024 \\
               & Optimizer & Adam \\
               & $\epsilon_1$  & kitchen 0.01, others 0.5 & coefficient for the KL regularizer \\
               & $\epsilon_2$ & kitchen 2, others 0.5 & coefficient for the KL regularizer \\
               &  $\alpha$ & 0.01 & mixing coefficient to the denominator of $R(s)$ \\
               &  $\beta$ & 5 & coefficient for combination of distance metric \\
               &  $\gamma$ & 0.998 & discount factor in our formulation \\
                  
        \hline
    \end{tabular}
    \caption{Our selection of hyperparameters. We use the same network architecture and optimizer as SMODICE~\cite{ma2022smodice}.}
    \label{tab:uniquehpp}
\end{table}

\section{Additional Experimental Results}
\label{sec:app-exp-res}
\subsection{Supplementary Results for Tabular Environment}
\label{sec:sup-tabular}
\subsubsection{State and State-pair Total Variation (TV) distance}

In this section, we show the Total Variation (TV) divergence between the state occupancies of the learner and the expert and the state-pair occupancies between the learner and the expert, i.e., $\text{TV}(d^\pi_{s}\|d^E_{s})$ and $\text{TV}(d^\pi_{ss}\|d^E_{ss})$. Fig.~\ref{fig:statediv-tabular} shows the result of the state occupancy distance between the learner's and the expert policies. Fig.~\ref{fig:statepairdiv-hard} shows the distance between the state-pair occupancies. We observe our method to work better than SMODICE and LobsDICE. 

\begin{figure}
    \centering
    \includegraphics[width=\linewidth]{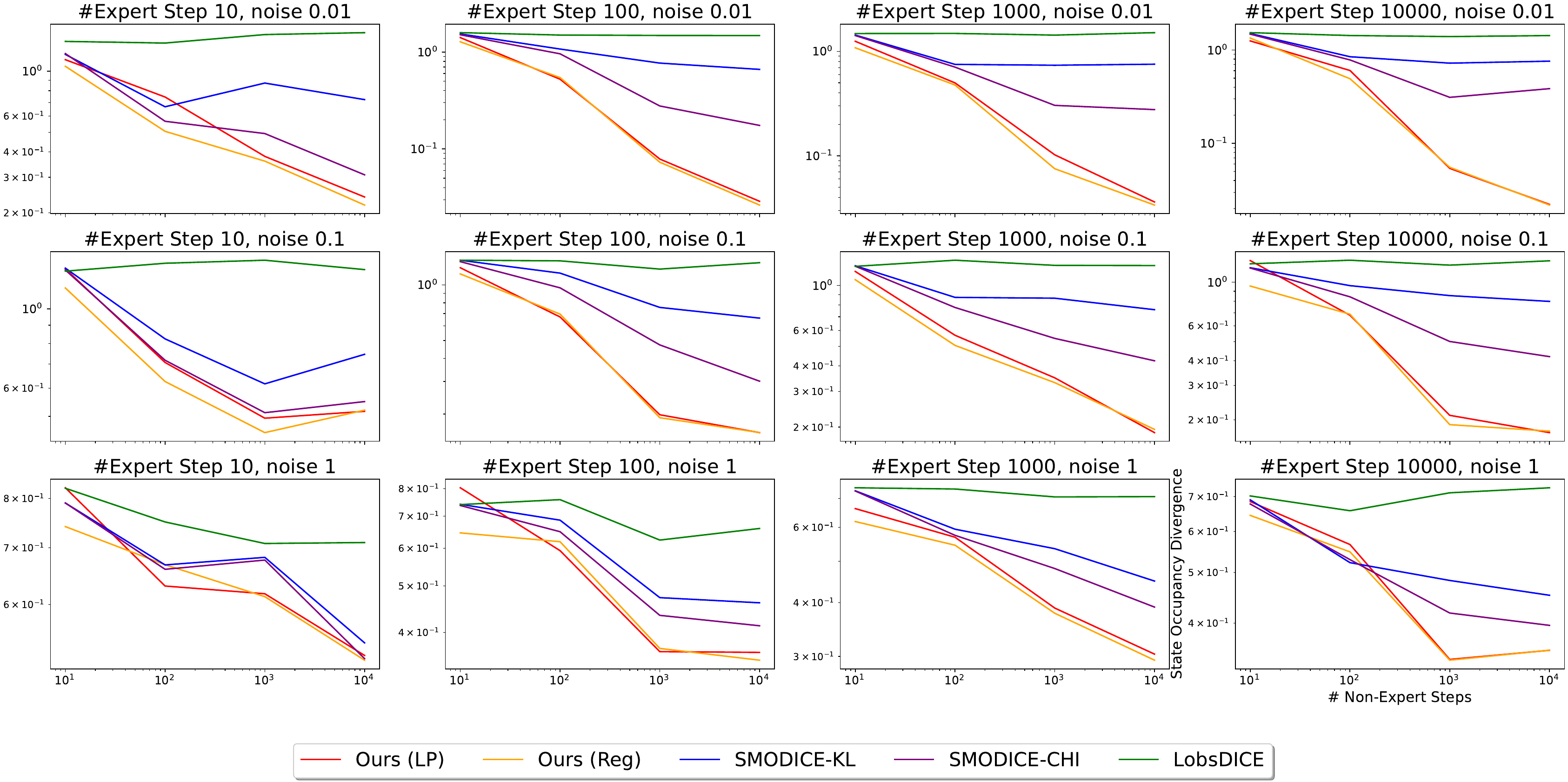}
    \caption{TV distance $\text{TV}(d^\pi_s\|d^E_s)$ of each method on tabular environments. Our method, both with and without regularizer, works comparably well as the baselines for a small task-agnostic dataset, and prevails with larger task-agnostic dataset (more accurate estimated dynamics).}
    \label{fig:statediv-tabular}
\end{figure}

\begin{figure}
    \centering
    \includegraphics[width=\linewidth]{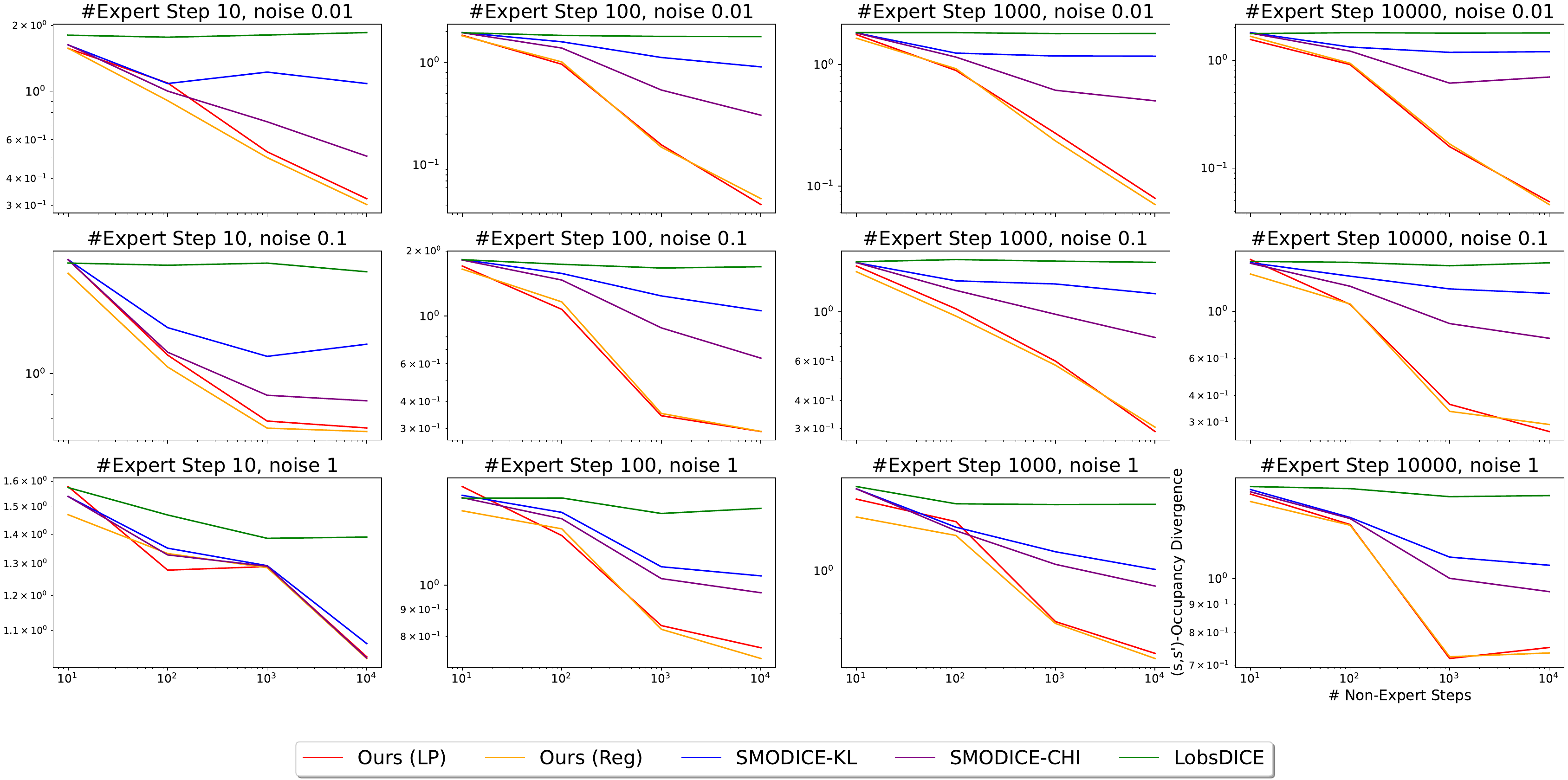}
    \caption{State-pair occupancy TV distance between the learner and expert ($\text{TV}(d^\pi_{ss}\|d^E_{ss})$) on tabular environments. Similar to TV distance between state-occupancies and regret, our method works the best, especially with larger task-agnostic dataset.}
    \label{fig:statepairdiv-hard}
\end{figure}

\subsubsection{Tabular Experiment with Softmax Expert}
\label{sec:softmax}
To be consistent with LobsDICE~\cite{Kim2022LobsDICEOL}, we also report experimental results using exactly the same setting of LobsDICE, which uses a highly sub-optimal expert. Fig.~\ref{fig:regret-soft}, Fig.~\ref{fig:statediv-soft}, and Fig.~\ref{fig:statepairdiv-soft} show the regret, state occupancy divergence $\text{TV}(d^\pi_s\|d^E_s)$, and state-pair occupancy divergence $\text{TV}(d^\pi_{ss}\|d^E_{ss})$ of each method in this setting respectively. The result shows that our method does not perform well in minimizing occupancy divergence, as the coefficient of the $f$-divergence regularizer in our PW-DICE is much smaller or $0$, which means that our obtained policy is more deterministic and thus different from the highly stochastic ``expert'' policy. It is worth noting that our method, with accurate estimation of MDP dynamics (i.e., large size of the task-agnostic/non-expert dataset), is the only method that achieves negative regret, i.e., our method is even better than the ``expert'' policy. Also, our method with regularizer generally achieves lower regret.

\begin{figure}[t]
    \centering
    \includegraphics[width=\linewidth]{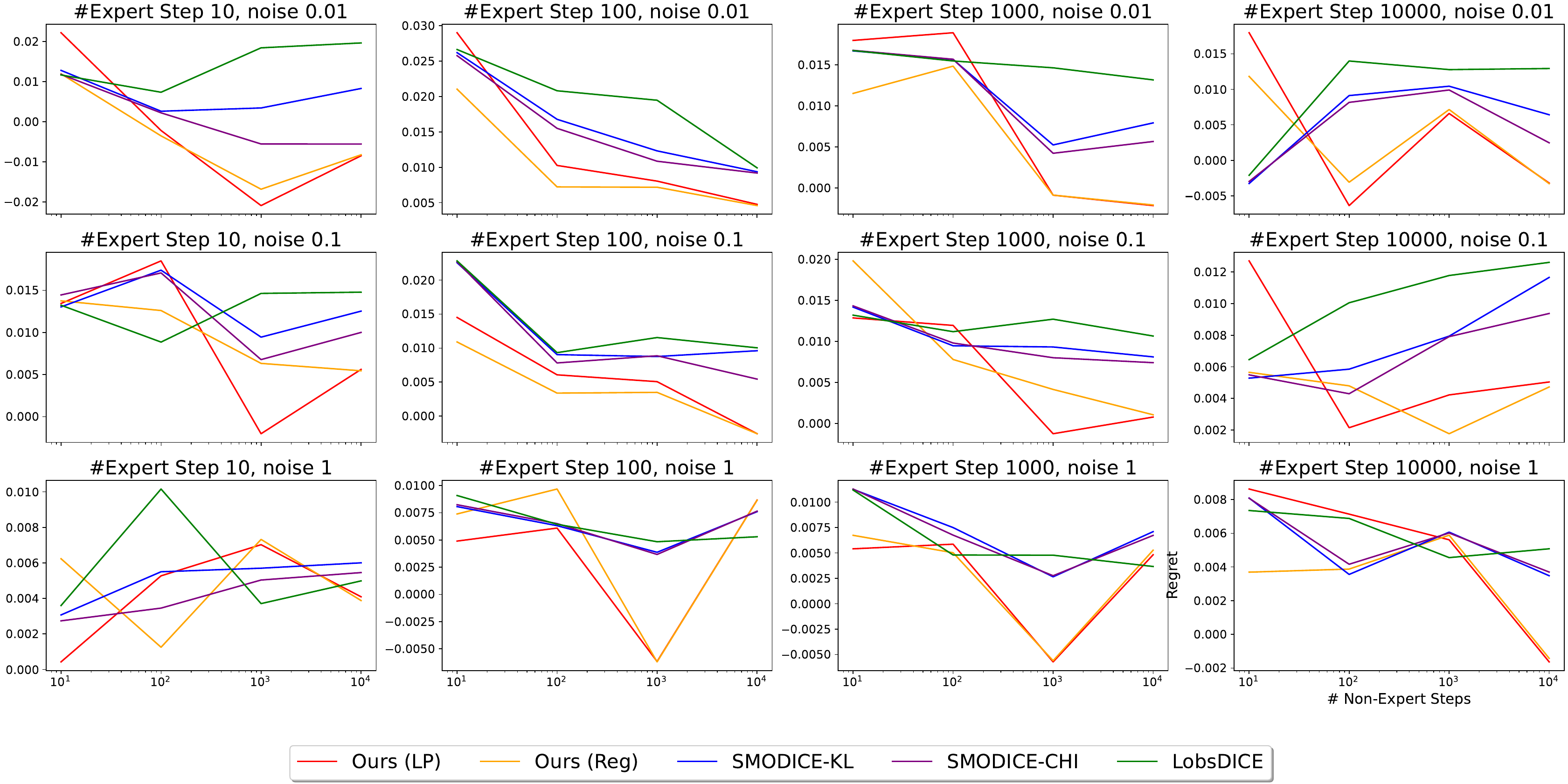}
    \caption{Regret of each method for tabular experiments with softmax expert. Our method with regularizer generally achieves the lower regret. Also, our method is the only one that achieves negative regret (i.e., better than the highly suboptimal ``expert'').}
    \label{fig:regret-soft}
\end{figure}

\begin{figure}[t]
    \centering
    \includegraphics[width=\linewidth]{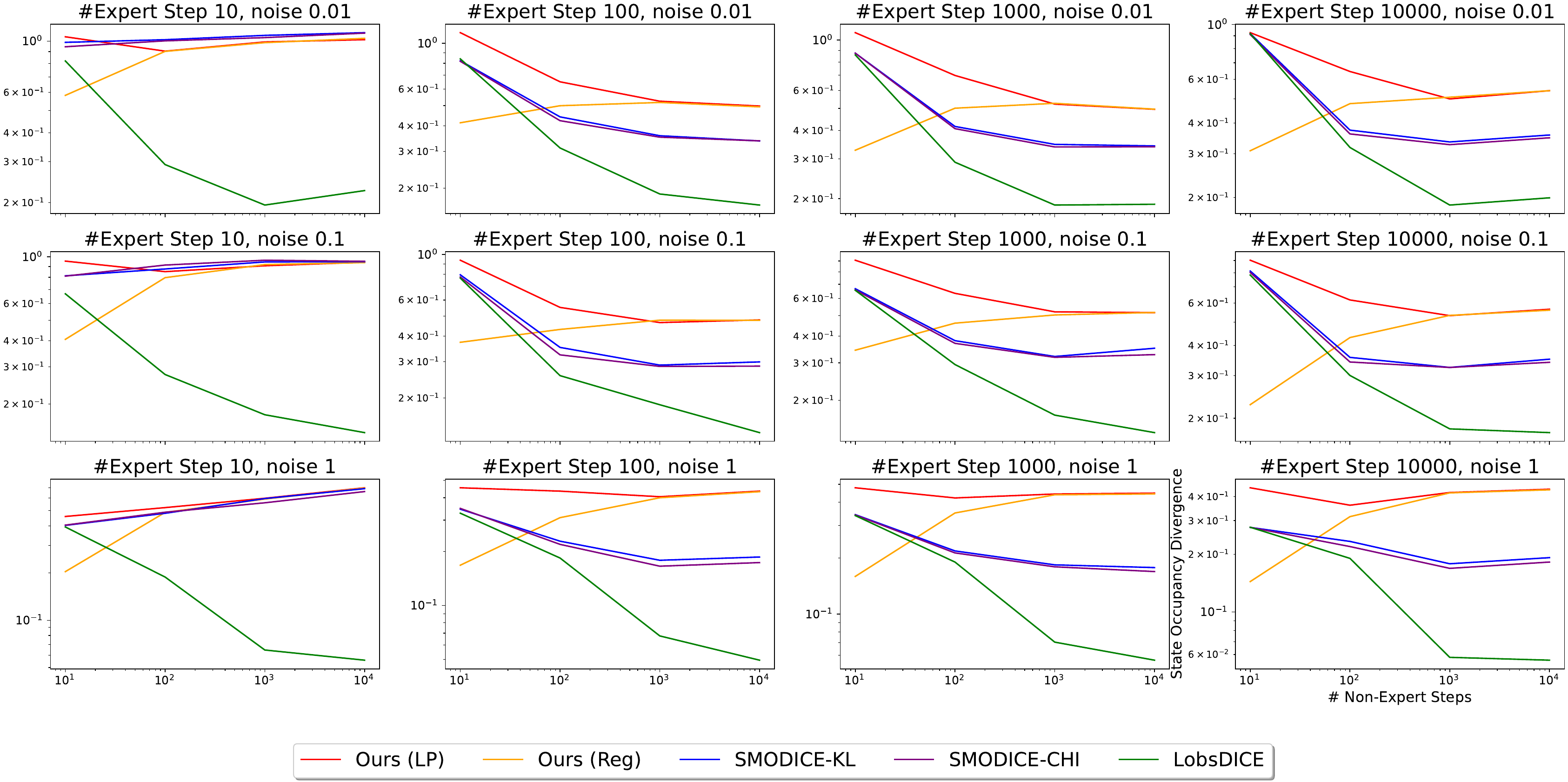}
    \caption{State occupancy TV distance $\text{TV}(d^\pi_s\|d^E_s)$ of each method for tabular experiments with softmax expert. Our method does not work well because the expert policy is highly stochastic.}
    \label{fig:statediv-soft}
\end{figure}

\begin{figure}[t]
    \centering
    \includegraphics[width=\linewidth]{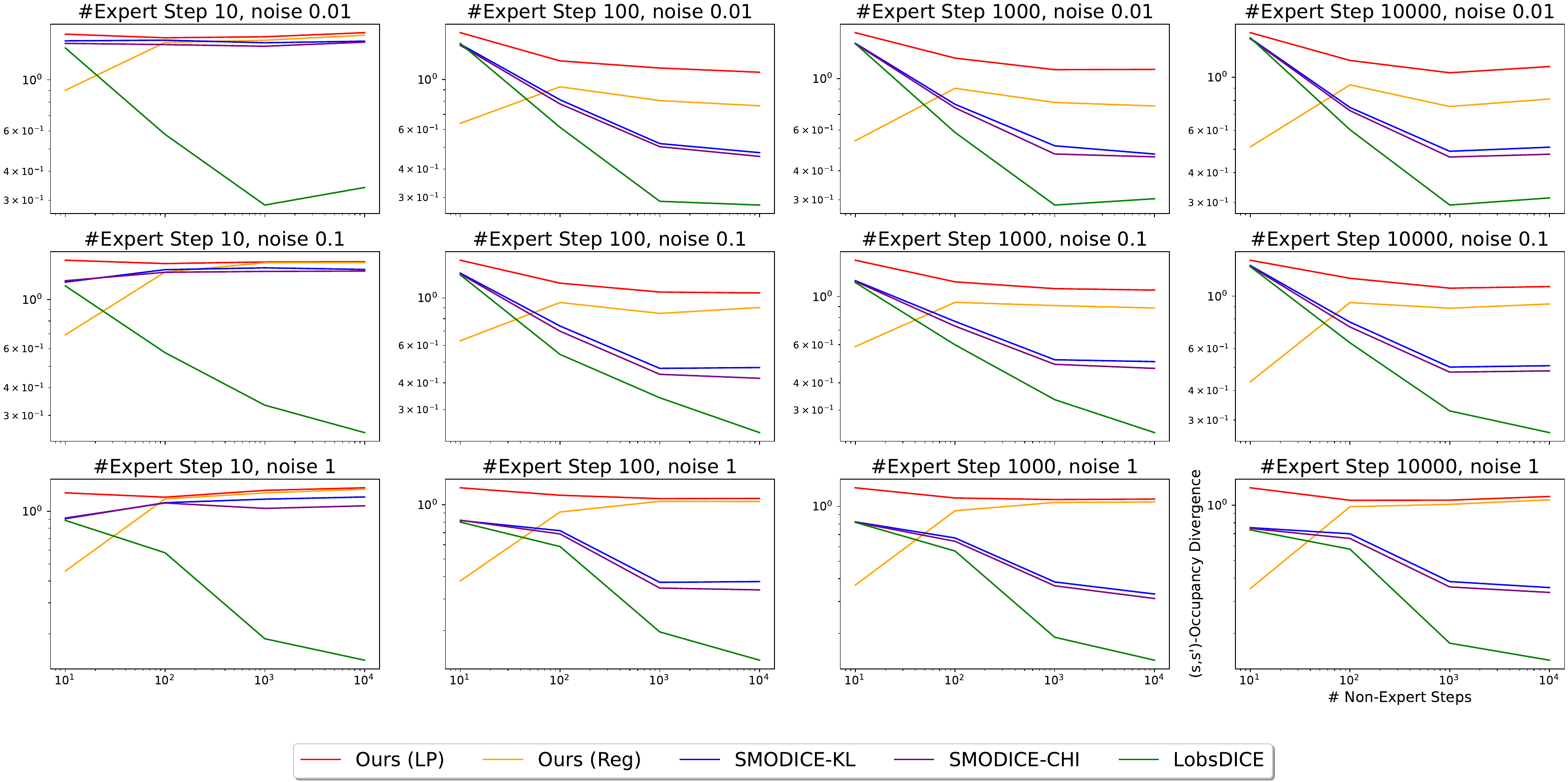}
    \caption{State-pair occupancy TV distance $\text{TV}(d^\pi_{ss}\|d^E_{ss})$ of each method for tabular experiments with softmax expert. Our method does not work well because the expert policy is highly stochastic.}
    \label{fig:statepairdiv-soft}
\end{figure}



\subsection{PW-DICE with $\chi^2$-divergence on MuJoCo Environment}
\label{sec:chidiv}
In the main paper, we mainly considered PW-DICE with KL-divergence. However, as Corollary~\ref{thm:corolapp} suggests, the $D_f$ regularizer in PW-DICE can also be $\chi^2$-divergence. Suppose we use half $\chi^2$-divergence as SMODICE~\cite{ma2022smodice} does, i.e., $f(x)=\frac{1}{2}(x-1)^2$, $f_*(x)=\frac{1}{2}(x+1)^2$, and $f'(x)=x+1$.
With such a divergence, the final optimization objective of PW-DICE reads as follows:

\begin{equation}
\begin{aligned}
&\min_{\lambda}\frac{\epsilon_1}{2}\mathbb{E}_{s_i\sim I, s_j\sim E}\left(\frac{\lambda_{i+|S|}+\lambda_{j+2|S|}-c(s_i,s_j)}{\epsilon_1}+1\right)^2\\+&\frac{\epsilon_2}{2}\mathbb{E}_{(s_i,a_j,s_k)\sim I}\left(\frac{-\gamma\lambda_k+\lambda_i-\lambda_{i+|S|}}{\epsilon_2}+1\right)^2-\left[(1-\gamma)\mathbb{E}_{s\sim p_0}\lambda_{:|S|}+\mathbb{E}_{s\sim E}\lambda_{2|S|:3|S|}\right],
\end{aligned}
\end{equation}

and the policy loss is

\begin{equation}
E_{(s,a)\sim I}\max\left(0, \frac{-\gamma\mathbb{E}_{s_k\sim p(\cdot|s_i,a_j)}\lambda_k+\lambda_i-\lambda_{i+
|S|}}{\epsilon_2}\right).
\end{equation}

However, similar to SMODICE, we found that the $\chi^2$-divergence regularizer does not work well under MuJoCo environments, as the weight ratio between good and bad actions in the task-agnostic dataset is only proportional (instead of exponential) to $-\gamma\lambda_k+\lambda_i-\lambda_{i+|S|}$, and thus is not discriminative enough. As a result, the retrieved policy is highly stochastic. Fig.~\ref{fig:ablation-chi} shows the result of $\chi^2$-divergence, which is much worse than the KL-divergence result.

\begin{figure}
    \centering
    \includegraphics[width=\linewidth]{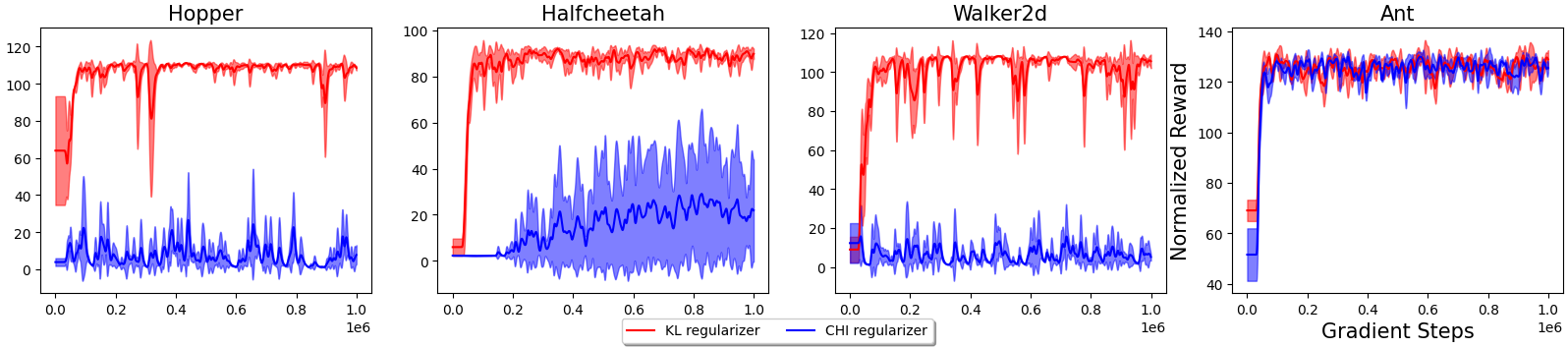}
    \caption{Performance comparison between $\chi^2$-divergence (blue) and KL-divergence (red) in PW-DICE. $\chi^2$-divergence does not work as well as KL-divergence.} 
    \label{fig:ablation-chi}
\end{figure}
{\color{black}
\subsection{Learning from Expert with Mismatched Dynamics}
In order to show that our method is robust with respect to different dynamics, we adopt the mismatched dynamics setting from SMODICE~\cite{ma2022smodice}, where the agent needs to learn from the same task-agnostic dataset as that in the main results of Sec.~\ref{sec:MuJoCoexp}, but with the expert dataset generated by an expert with very different dynamics; for example, one of the legs of the expert is amputated in the ant environment, and the torso of the expert is much shorter in the halfcheetah environment. We use exactly the same setting as SMODICE; Fig.~\ref{fig:mismatch} shows the result, which illustrates that our method is generally more robust to embodiment differences than SMODICE, LobsDICE, and ORIL.

\begin{figure}
    \centering
    \includegraphics[width=\linewidth]{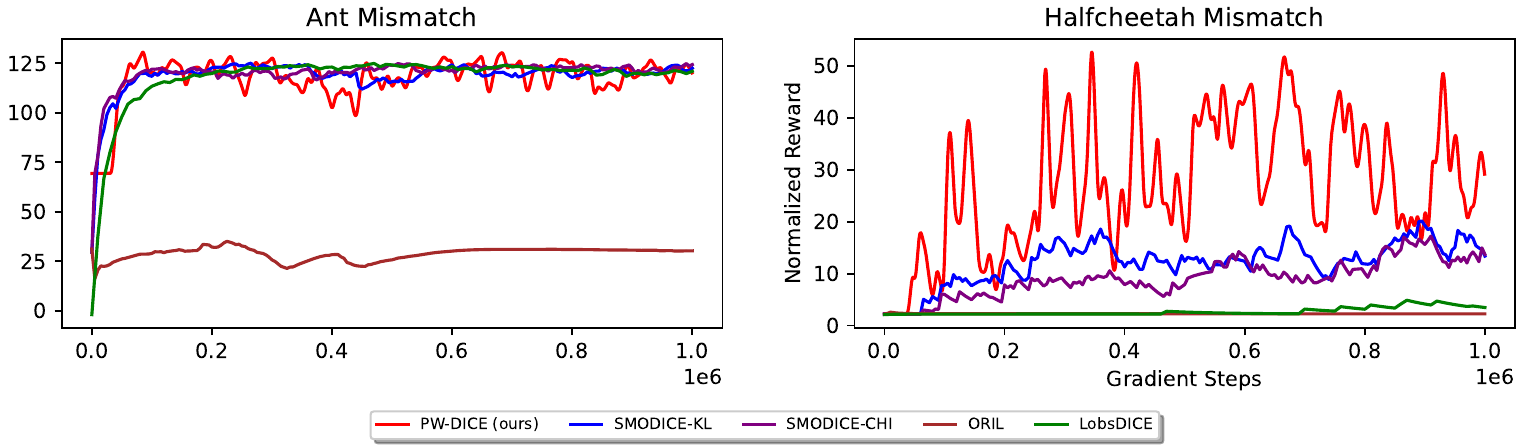}
    \caption{\color{black}Performance comparison on expert with mismatched dynamics; it is shown that our method is generally more robust than SMODICE, LobsDICE, and ORIL on the two environments.}
    \label{fig:mismatch}
\end{figure}

}
\section{Notations Table}
\label{sec:notlist}

Tab.~\ref{tab:notelist} lists the symbols that appear in the paper.

\begin{table}[t]
    \centering
    {\small
    \begin{tabular}{ccp{8cm}}
       \hline
        Name & Meaning & Note \\
        \hline
        $S$ & State space & $|S|$ is the size of state space\\
        $A$ & Action space & $|A|$ is the size of state space\\
        $\gamma$ & Discount factor & $\gamma\in(0, 1)$\\
        $r$ & Reward function &  $r(s,a)$ for single state-action pair \\
        $T$ & Transition function & \\
        $p$ & Transition (single entry) & $p(s'|s,a)\in\Delta(S)$ \\
        $p_0$ & Initial distribution & $p_0\in\Delta(S)$ \\
        $s$ & State & $s\in S$\\
        $a$ & Action & $a\in A$ \\
        $\bar{s}$ & Past state & \\
        $\bar{a}$ & Past action & \\
        $\tau$ & Trajectories & State-only or state-action; depend on context\\
        $E$ & Expert dataset & state-only expert trajectories \\
        $I$ & Task-agnostic dataset & state-action trajectories of unknown optimality \\
        $\pi$ & Learner policy & \\
        $\pi^E$ & Expert policy abstracted from $E$\\
        $\pi^I$ & Task-agnostic policy abstracted from $I$\\
        $d^\pi_{sa}$ & State-action occupancy of $\pi$ & \\
        $d^\pi_{s}$ & State occupancy of $\pi$ & 1) $\forall s\in\mathcal{S}, \sum_a d^\pi_{sa}(s,a)=d^\pi_s(s)$. This equation also applies similarly between $d^E_{sa}$ and $d^E_{s}$, as well as $d^I_{sa}$ and $d^I_{s}$. \newline
        2) $d^\pi_s(s)=(1-\gamma)\sum_{i=0}^\infty\gamma^i\Pr(s_i=s)$, where $s_i$ is the $i$-th state in a trajectory. This holds similarly for $d^I(s)$ and $d^E(s)$.\newline
        3) $d^\pi_{sa}(s,a)=d^\pi_s(s)\pi(a|s)$. This holds similarly for $d^E_{sa}, \pi_E$ and $d^I_{sa}, \pi_I$.\\
        $d^\pi_{ss}$ & State-pair occupancy of $\pi$ & \\
        $d^E_{s}$ & State occupancy of $\pi^E$\\ 
        $d^E_{ss}$ & State-pair occupancy of $\pi^E$ & \\
        $d^I_{sa}$ & State-action occupancy of $\pi^I$\\
        $d^I_{s}$ & State occupancy of $\pi^I$\\
        $\lambda$ & Dual variable \\
        $D_f$ & $f$-divergence \\
        $f_*(\cdot)$ & Fenchel conjugate of $f$ \\ 
        \hline
        $c$ & Matching cost for Wasserstein distance \\
        $c'$ & Matching cost for Wasserstein distance & With extended domain \\
        $\Pi$ & Wasserstein matching variable & $\sum_{s\in S}\Pi(s,s')=d^E_s(s')$, $\sum_{s'\in S}\Pi(s,s')=d^\pi_s(s)$\\
        $A$ & Equality constraint matrix \\
        $x$ & unified self-variable & concatenation of flattened $\Pi$ and $d^\pi_{sa}$ (row first) \\
        $b$ & Equality constraint vector & $Ax=b$ \\
        $U$ & Distribution as regularizer & product of $d^I_s$ and $d^E_s$\\
        $\mathcal{W}$ & Wasserstein distance\\
        {\color{black} $h(\cdot)$} & {\color{black}state discriminator} \\
        \hline
        $W$ & Weight matrix of contrastive learning & $W\in\mathbb{R}^{32\times 32}$\\
        $g(\cdot)$ & embedding to be learned & \\
        $L$ & score matrix & $L\in\mathbb{R}^{4096\times 4096}$\\
        $n$ & number of dimensions for state & \\
        $M$ & number of dimensions for embedding & $M=32$ \\
        {\color{black}$\beta$} & {\color{black}coefficient for learned embedding in distance metric}
    \end{tabular}
    \newline
    \caption{Complete list of notations used in the paper. The first part is for offline LfO settings, the second part lists notations specific to PW-DICE, and the third part is for notations used in contrastive learning (Appendix~\ref{sec:mjc}).}
    \label{tab:notelist}
    }
\end{table}

\section{Computational Resources}
\label{sec:resource}

All experiments are carried out with a single NVIDIA RTX 2080Ti GPU on an Ubuntu 18.04 server with 72 Intel Xeon Gold 6254 CPUs @ 3.10GHz. Given these resources, our method needs about $5-5.5$ hours to finish training in the MuJoCo environments (during which the training of the distance metric, including $R(s)$ and contrastive learning, takes $20-40$ minutes), while ORIL, SMODICE, and LobsDICE require about $2.5-3$ hours. As the actor network is identical across all methods, the inference speed is similar and is not a bottleneck.

\end{document}